%% file: NeurIPS_arXiv.tex
\documentclass{article}

\usepackage[preprint, nonatbib]{neurips_2023}

\usepackage[utf8]{inputenc} 
\usepackage[T1]{fontenc}    
\usepackage{hyperref}       
\usepackage{url}            
\usepackage{booktabs}       
\usepackage{amsfonts}       
\usepackage{nicefrac}       
\usepackage{microtype}      
\usepackage{xcolor}         

\usepackage{xspace, bm, bbm, enumitem, color, caption, subcaption}
\usepackage{amsmath, amssymb, mathtools, amsthm}
\input{macros.tex} 
\usepackage{graphicx}
\graphicspath{{figs/}}
\newcommand{\ReLUmax}[3]{#1 \stackrel{\sigma}{\rightarrow} #2 \stackrel{\sigma}{\rightarrow} #3 \stackrel{\texttt{MAX}}{\rightarrow} 1}
\newcommand{\ReLUtwo}[3]{#1 \stackrel{\sigma}{\rightarrow} #2 \rightarrow #3}
\newcommand{\ReLUTwo}[3]{#1 \stackrel{\sigma}{\rightarrow} #2 \stackrel{\sigma}{\rightarrow} #3}
\newcommand{\ReLUThree}[4]{#1 \stackrel{\sigma}{\rightarrow} #2 \stackrel{\sigma}{\rightarrow} #3 \stackrel{\sigma}{\rightarrow} #4}
\newcommand{\ReLUFour}[5]{#1 \stackrel{\sigma}{\rightarrow} #2 \stackrel{\sigma}{\rightarrow} #3 \stackrel{\sigma}{\rightarrow} #4 \stackrel{\sigma}{\rightarrow} #5}

\title{Data Topology-Dependent Upper Bounds of \\ Neural Network Widths} 

\author{%
  Sangmin Lee$^1$ \hspace{2cm} Jong Chul Ye$^{2}$ \vspace{0.1cm}\\
  $^1$ Department of Mathematical Science, KAIST \\
  $^2$ Kim Jaechul Graduate School of AI, KAIST \\
  \texttt{\{leeleesang, jong.ye\}@kaist.ac.kr}
}

\begin{document}
\allowdisplaybreaks
\maketitle

\begin{abstract}
    This paper investigates the relationship between the universal approximation property of deep neural networks and topological characteristics of datasets. 
    Our primary contribution is to introduce data topology-dependent upper bounds on the network width.
    Specifically, we first show that a three-layer neural network, applying a ReLU activation function and max pooling, can be designed to approximate an indicator function over a compact set, one that is encompassed by a tight convex polytope.
    This is then extended to a simplicial complex, deriving width upper bounds based on its topological structure.
    Further, we calculate upper bounds in relation to the Betti numbers of select topological spaces.
    Finally, we {prove} the universal approximation property of three-layer ReLU networks using our topological approach.
    We also verify that gradient descent converges to the network structure proposed in our study.
\end{abstract}

\section{Introduction}

This paper addresses a fundamental question in machine learning: for any $p\ge1$, $d\in \Nd$, and $f^* \in L^p([0,1]^d)$, what is the necessary depth and width of a neural network to approximate $f^*$ within a small error? This constitutes the universal approximation property (UAP) of deep neural networks (DNNs). Since Cybenko's seminal work in 1989  \cite{cybenko1989approximation}, demonstrating the UAP of two-layer networks with non-polynomial activation functions, subsequent research has extended these results to various settings.

In the context of deep ReLU networks, recent literature establishes that the minimal depth is $2$ (given sufficient width), and the minimal width is $\max\{d_x, d_y\}$ assuming adequate depth  over a compact domain for some function classes {where $d_x$ and $d_y$ are input and output dimensions} \cite{hornik1991approximation, park2020minimum}. However, these UAP results on a compact domain have limitations for classifiers or discriminators. For instance, a discriminator used in generative adversarial network (GAN) training receives both training data and output from a generator, which is not confined to a specific bounded domain \cite{hanin2019universal, hertrich2021towards, hwang2022minimal, kratsios2022relu, perekrestenko2018universal, perekrestenko2020constructive}.

There are recent UAP results for neural networks concerning unbounded input domains (like $\Rd^d$). Notably, a study by Wang et al. \cite{wang2022approximation} demonstrates that two-layer ReLU networks fail to serve as universal approximators on $\Rd^2$. 
As such, to approximate compactly supported functions in $\Rd^d$, the required minimum depth is at least $3$. While Wang et al. \cite{wang2022approximation} affirmed that three-layer ReLU networks are universal approximators in $L^p(\Rd^d)$, they only establish the existence of such networks, leaving open the questions of their construction and the number of required hidden neurons. This paper addresses these questions specifically for the class of discriminator functions.

Specifically, we explore the following question: 
{\em given dataset $\Xc$, $\varepsilon>0$ and $p\ge1$, how can we construct a neural network $\Nc$ such that $\norm{\Nc(\xb) - \indicator{\Xc}(\xb)}_{L^p(\Rd^d)} < \varepsilon$? 
}
Intuitively, the answer is closely tied to the topological structure of the dataset $\Xc$. The manifold assumption in machine learning suggests that dataset $\Xc$ adheres to a low-dimensional manifold, representable by simple topological features  \cite{carlsson2014topological, goodfellow2016deep, lum2013extracting, magai2022topology, nicolau2011topology}.
Consequently, if $\Xc$ has a `simple' topological structure, the required depth and width to approximate $\indicator{\Xc}$ would be minimal. 
Numerous experimental and theoretical results indeed show a strong correlation between the required width of neural networks and the topology of training dataset $\Xc$ \cite{fort2022does, guss2018characterizing, magai2022topology, tiwari2022effects, trofimov2023learning, yang2021dive}. Despite this, no research has specifically addressed how to architect neural networks based on the topological structure of the dataset. 
{For instance, Betti numbers in Topological data analysis (TDA)
represent the topological structure of dataset $\Xc$, but no previous work has linked this quantity with network architecture.}

Therefore, this paper is to bridge this gap by constructing three or four-layer neural networks with bounded widths determined by the topological features of the dataset $\Xc$.
Our contributions are summarized below.
\begin{itemize}
    
    \item  Motivated by the results of Wang et al. \cite{wang2022approximation}, {we generalize their negative result for two-layer ReLU networks on $\Rd^d$ with $d \ge 2$. The proof is in Appendix \ref{app: proofs}.}
    \begin{proposition} \label{prop: two-layer cannot}
        Let $f^*:\Rd^d \rightarrow \Rd$ be a nonzero compactly supported function. Then for $p\ge1$ and $d \ge 2$, two-layer ReLU networks cannot universally approximate $f^*$ in $L^p(\Rd^d)$.
    \end{proposition} 

    \item 
    For a compact set $\Xc\subset \Rd^d$, we develop a three-layer neural network $\Nc$ approximating the indicator function $\indicator{\Xc}$ over $\Rd^d$ under a given $\varepsilon$ error. The network's width is bound by a number related to a convex polytope cover of $\Xc$ (as stated in Proposition \ref{prop: compact}). We further refine this result for a four-layer ReLU network when $\Xc$ is represented by the difference of two unions of convex polytopes (Theorem \ref{thm: compact}).
    
    \item
    When $\Xc$ forms an $m$-simplicial complex with $k$ facets in $\Rd^d$, we derive a similar result. Here, we establish upper bounds of the width in terms of $d, k$, and $m$, introducing novel data topology-dependent bounds (Theorem \ref{thm: simplicial complex}). If $\Xc$ is a topological space represented by the difference of $k$-dimensional disjoint cuboids, we propose a four-layer ReLU network that can approximate $\indicator{\Xc}$, {where the widths are} bound by terms of the Betti number of $\Xc$ (Theorem \ref{thm: betti numbers}). This underscores the significant impact of Betti numbers on the network architecture, a novel contribution in this field.
    
    \item 
    As a practical application, we {prove} that the set of three-layer ReLU networks is dense in $L^p(\Rd^{d_x}, [0,1]^{d_y})$ for $p\ge 1$, confirming the UAP of three-layer ReLU networks over an unbounded domain (Theorem \ref{thm: regression}). In conjunction with Proposition \ref{prop: two-layer cannot}, this confirms that the minimal depth of deep ReLU networks in $L^p(\Rd^{d \ge 2})$ is exactly $3$. Furthermore, {for a Lipschitz function} $f^*: [0,1]^{d_x} \rightarrow [0,1]^{d_y}$, we offer upper bounds of width by $O(\varepsilon^{-d_x})$, where $\varepsilon$ represents the  error bound.
\end{itemize}

\vspace{-0.2cm}
\section{Related Works} \label{sec: related}
In this section, we will review studies that relate to topological approaches in deep learning, compactly supported DNNs, and the universal approximation property (UAP) along with width bounds for small depth neural networks.
\vspace{-0.2cm}
\paragraph{Topological approach in deep learning.}
Although few studies have explored the connection between neural network architecture and the topological features of the training data, their findings hold significant implications. \cite{naitzat2020topology} conducted experiments demonstrating the rapid reduction of Betti numbers of the data topology during training with deep ReLU layers. \cite{guss2018characterizing} attempted to correlate the architecture and data topology by increasing the complexity of the training data to guide the selection of deep neural network architecture. \cite{yang2021dive} offered upper bounds on the Betti numbers for the preimage of a layer in deep neural networks and compared them with the Betti numbers of the dataset to inform network architecture selection. Further experiments have underscored the relevance of topology-dependent neural network architecture \cite{hajij2020topological, hajij2021topological, magai2022topology, tiwari2022effects}. For an in-depth review of topological deep learning, refer to \cite{hensel2021survey, zia2023topological}.

\paragraph{Compactly supported DNNs.} 
The construction of compactly supported subnetworks for use as building blocks in DNNs has been a focus in several studies, often with the aim of approximating a compactly supported function \cite{barbu2021compact, he2018relu, song2023approximation}. For instance, \cite{huang2020relu} proposed a `TRLU function' for approximating constant or piecewise linear functions on a compact domain, and \cite{huang2022theoretical} further examined the resulting neural networks, considering open convex polytopes as the partition of input space. {This concept aligns closely with Lemma \ref{lem: convex polytope} which we used to construct desired neural networks.}
Additionally, \cite{kratsios2022relu} studied the UAP for deep neural networks with ReLU and pooling layers over compactly-supported integrable functions, providing bounds of width, depth, and the number of pooling layers. Our work aligns with these studies as we explicitly construct small depth (three or four-layer) neural networks using ReLU activation and pooling layers to approximate the indicator function over a given topological space.

\paragraph{UAP and width bounds for small depth neural networks.}

The Universal Approximation Property (UAP) of shallow neural networks and the width bounds associated with them have been a subject of extensive study under various conditions \cite{barron1993universal, hecht1987kolmogorov, ismailov2020three, pinkus1999approximation}. Recently, \cite{hertrich2021towards} utilized tropical geometry \cite{maclagan2021introduction} and polyhedral theory to prove the UAP of two-layer neural networks, a method analogous to our approach in this paper. They also provided width and depth bounds to approximate a function that is the difference of two convex functions, a result we replicate for indicator functions over the difference of unions of convex polytopes.

Addressing unbounded domain $\Rd^d$, \cite{eldan2016power} proposed a radial function that two-layer networks fail to approximate. \cite{wang2022approximation} further contributed to this dialogue by presenting both negative and positive results on unbounded domains. They demonstrated that two-layer ReLU neural networks cannot universally approximate on the Euclidean plane $\Rd^2$. However, they also proved that three-layer ReLU networks can serve as universal approximators in $L^p(\Rd^d)$. Drawing inspiration from these findings, we extend their initial result by illustrating that two-layer ReLU networks cannot serve as universal approximators on $L^p(\Rd^d)$ for $p\ge1$ and $d\ge2$ (Proposition \ref{prop: two-layer cannot}).

On a different note, \cite{cohen2019universal} deduced that any Lipschitz function over a compact set can be approximated by a three-layer neural network with a novel activation function. This research bears similarity to our result in Section \ref{sec: regression}, albeit with a slightly different conclusion. While they introduced a new activation function and focused on the boundedness of matrix norm, we exclusively employ the ReLU activation function and provide the upper bound of widths. Additionally, \cite{kratsios2022relu} explored the UAP of DNNs approximating compactly supported Lipschitz functions through ReLU and pooling layers.

In this paper, we merge these approaches to establish data-topology dependent width bounds and to validate the UAP of three-layer ReLU networks on $\Rd^d$.

\section{Data-Topology Dependent Upper Bounds of Widths} \label{sec: main}
\vspace{-0.2cm}
\subsection{Preliminaries}
\paragraph{Notation.}

In this article, scalars are denoted by lowercase letters, vectors by boldface lowercase letters, and matrices by boldface uppercase letters.
The $L^p$-norm in function spaces is represented as $\norm{\cdot}_{L^p}$.
For a positive integer $m$, $[m]$ is used to represent the set $\{ 1,2,\cdots,m \}$.
The ReLU activation function is denoted by $\sigma(x):= \textup{ReLU}(x)=\max\{0, x \}$, and it is applied to a vector coordinate-wise.
The sigmoid activation function is denoted as $\texttt{SIG}(x) = \frac{1}{1+e^{-x}}$.
The max pooling operation is represented as $\texttt{MAX} : \Rd^d \rightarrow \Rd$, returning the maximum value among the elements of a given vector.
$B_\varepsilon(\xb_0):=\{\xb : \norm{\xb - \xb_0}_2 <\varepsilon\}$ denotes the epsilon neighborhood of $\xb_0$. The $\varepsilon$ neighborhood of a compact set $\Xc$ is defined by $B_\varepsilon(\Xc):=\{\xb : \min_{\yb\in\Xc} \norm{\xb - \yb}_2 <\varepsilon\}$.
Lebesgue measure in $\Rd^d$ is represented by $\mu_d$, or simply $\mu$ when $d$ is apparent in the context.
For a given set $\Xc\subset\Rd^d$, the indicator function over $\Xc$ is denoted as follows:
\begin{align*}
    \indicator{\Xc}(\xb) := 
    \begin{cases}
    	1, \qquad \text{if } \xb\in\Xc , \\ 
    	0, \qquad \text{otherwise.}
    \end{cases}
\end{align*}

\paragraph{Deep neural networks.}

We use the following notation to represent the architecture of a $k$-layer neural network $\Nc:\Rd^d \rightarrow \Rd$ with widths $d_1, d_2, \cdots, d_{k-1}$ and activation functions $\texttt{ACT}_1,  \texttt{ACT}_2, \cdots, \texttt{ACT}_{k}$ on each hidden layer: $d\stackrel{\texttt{ACT}_1}{\rightarrow}d_1\stackrel{\texttt{ACT}_2}{\rightarrow}d_2\stackrel{\texttt{ACT}_3}{\rightarrow} \cdots \stackrel{\texttt{ACT}_{k-1}}{\rightarrow} d_{k-1} \stackrel{\texttt{ACT}_k}{\rightarrow} 1$. When the activation function is identity, we denote nothing on the arrow.
For example, consider a three-layer fully connected network $\Nc$ with ReLU activation functions on hidden layers and a max pooling operation for the last layer, defined by
\begin{align*}
    \Nc(\xb) = \texttt{MAX} \left[ \sigma(\Wb_2\sigma(\Wb_1 \xb + \bb_1)+\bb_2) \right].
\end{align*}
{Then} its architecture is denoted by $\ReLUmax{d}{d_1}{d_2}$. 

The objective of this paper is to construct a neural network $\Nc$ such that, given $\varepsilon>0$, $p\ge1$, and a compact set $\Xc\subset\Rd^d$, the following inequality holds: $\norm{\Nc(\xb) - \indicator{\Xc}(\xb)}_{L^p(\Rd^d)}<\varepsilon$. By Proposition \ref{prop: epsilon}, it is sufficient to construct a neural network $\Nc$ that satisfies the following {three} conditions:
\textbf{A.} $\Nc(\Rd^d)\subset[0,1]$,
\textbf{B.} $\Nc(\xb) = 1$ if $\xb \in \Xc$, and
\textbf{C.} $\Nc(\xb) = 0$ if $\xb \not\in B_{\varepsilon'}(\Xc)$, for a given $\varepsilon'>0$.
In other words, the desired network should output a constant value of $1$ over the given manifold $\Xc$ and vanish for inputs farther than $\varepsilon'$ from $\Xc$. This is a property desired for classifiers or discriminators.

\subsection{Main Theoretical Findings}

\subsubsection{Upper bounds of the widths for compact sets}

Consider $\Xc \subset \Rd^d$, a compact set, and the task of approximating its indicator function $\indicator{\Xc}$ within $\Rd^d$. As per Proposition \ref{prop: two-layer cannot}, a minimum of three layers is required for this task. Intriguingly, if $\Xc$ can be encapsulated within a collection of convex polytopes, a three-layer neural network employing ReLU activation and max pooling operations can be constructed. The following proposition outlines not just the feasibility of such a neural network, but also provides a method for its construction.

\begin{proposition} \label{prop: compact}
    Let $\Xc\subset\Rd^d$ be a compact set. For a given $\varepsilon>0$, suppose there exists a finite collection of convex polytopes $\Cc$ such that $\Xc \subset \bigcup_{C\in\Cc}C \subset B_\varepsilon(\Xc)$.
    Let $k:=|\Cc|$ be the caldinality of $\Cc$, and $l$ be the total sum of number of faces of each polytope in $\Cc$. Then, there exists a three-layer neural network $\Nc$ with the architecture $\ReLUmax{d}{l}{k}$ such that $\Nc(\Rd^d)=[0,1]$ and
    \begin{align*}
        \Nc(\xb) &= 1 \qquad \text{if } \xb \in \Xc, \\
        \Nc(\xb) &= 0 \qquad \text{if } \xb \not\in B_\varepsilon(\Xc).
    \end{align*}
\end{proposition}
\begin{proof} \vspace{-0.2cm}
    Let $\Cc = \{C_1, \cdots, C_k\}$ and $l_i$ be the number of faces of
    $C_i$. By Lemma \ref{lem: convex polytope}, for each $C_i\in\Cc$, there exists a two-layer ReLU network $\Tc_i$ with the architecture $\ReLUtwo{d}{l_i}{1}$ such that $\Tc_i(\xb)=1$ for $\xb \in C_i$ and $\Tc_i(\xb)<0$ for $\xb \not\in B_\varepsilon(C)$.
    Therefore, $\sigma(\Tc_i(\xb))=1$ for $\xb \in C_i$ and $\sigma(\Tc_i(\xb))=0$ for $\xb \not\in B_\varepsilon(C_i)$.
    Now, we take max pooling operation to define a three-layer neural network $\Nc$. 
    \begin{align*}
        \Nc(\xb):=\texttt{MAX}(\Tc_1(\xb), \cdots, \Tc_k(\xb)).
    \end{align*}
    Then it is easy to verify that $\Nc(\xb)$ is the desired three-layer neural network which has the architecture $\ReLUmax{d}{l}{k}$, where $l=l_1+\cdots+l_k$.
\end{proof}

Proposition~\ref{prop: compact}  confirms the universal approximation property of three-layer neural networks for indicator functions over a compact set, granted it can be closely covered by finite convex polytopes. It is crucial to highlight that the upper bound on the neural network's widths is dictated by the covering $\Cc$, specifically the number of constituents $k$ and the cumulative number of all faces $l$.
Given that a single compact set $\Xc$ may have multiple potential convex polytope coverings, deciding on the optimal covering method becomes a significant consideration. If $\varepsilon$ is significantly large, we can opt for a loose cover of $\Xc$ with smaller values of $k$ and $l$. However, a smaller $\varepsilon$ necessitates larger values for both $k$ and $l$.

This brings us to an extension of the original proposition to address the issue of a high count of convex polytopes in the covering of $\Xc$, which can occur due to the set's intricate characteristics. The subsequent theorem tackles this problem by increasing depth: if $\Xc$ can be enveloped by the difference between two unions of convex polytopes, then a four-layer ReLU network can effectively approximate $\indicator{\Xc}$.

\begin{theorem} \label{thm: compact}
    Let $\Xc\subset \Rd^d$ be a compact set.
    Suppose there exists a finite collection of convex polytopes $\Cc = \{P_1,\cdots, P_{n_{P}}, Q_1,\cdots, Q_{n_Q}\}$ such that the set difference
    $D:= \bigcup_{i\in[n_P]} P_i - \bigcup_{j\in[n_Q]} Q_j$ satisfies
    $\Xc\subset D \subset B_\varepsilon(\Xc)$. Let $l$ denote the total number of faces of the convex polytopes in $\Cc$.
    Then, there exists a four-layer ReLU network $\Nc$ with the architecture $\ReLUFour{d}{l}{(n_P+n_Q)}{2}{1}$ such that $\Nc(\Rd^d) = [0,1]$ and
    \begin{align*}
        \Nc(\xb) &= 1 \qquad \text{if } \xb\in\Xc, \\
        \Nc(\xb) &= 0 \qquad \text{if } \xb\not\in B_\varepsilon(\Xc).
    \end{align*} 
\end{theorem}
\begin{proof}\vspace{-0.2cm}
    By Lemma \ref{lem: convex polytope}, for each $A\in \Cc=\{ P_1 ,\cdots, P_{n_P}, Q_1, \cdots, Q_{n_Q}\}$, we can construct a two-layer ReLU network $\Tc_A$ such that $\Tc_A(\xb) = 1$ for $\xb\in A$ and $\Tc_A(\xb)=0$ for $\xb\not\in B_\varepsilon(A)$. Let $a_i := \Tc_{P_i}$ for $i\in [n_P]$ and $b_j := \Tc_{Q_j}$ for $j\in [n_Q]$. 
    Define two neurons in the third layer by 
    \begin{align*}
        a:= \sigma(1-a_1-\cdots-a_{n_P}) \qquad\text{and}\qquad b:=\sigma(1-b_1-\cdots-b_{n_Q}).
    \end{align*}
    Finally, defining the last layer by $\sigma(b-a)$, we obtain the desired network $\Nc$.
    The architecture of this network is $\ReLUFour{d}{l}{(n_P+n_Q)}{2}{1}$.
\end{proof}

The primary advantage of this theorem is the reduction in the width of neural networks. While Proposition \ref{prop: compact} requires the covering of $\Xc$ solely through the union of convex polytopes, Theorem \ref{thm: compact} relaxes this requirement by allowing the difference between two unions of convex polytopes. This can decrease the necessary number of neurons, given that a compact set might be covered {by the difference of two unions of fewer convex polytopes}, as we will illustrate in Example \ref{exp: k gon} (Figure \ref{fig: polygon}).

However, a persisting challenge is that there is no general method known for covering a compact set using convex polytopes, or their differences. To address this, in the following section, we present general upper bounds of widths in three-layer neural networks when $\Xc$ forms a simplicial complex.

\subsubsection{Upper bounds of widths for simplicial \texorpdfstring{$m$}{m}-complexes}

Before delving into the theorem, {we recall} some definitions. A simplicial $m$-complex is a type of simplicial complex where the highest dimension of any simplex equals $m$. For a given simplicial complex $K$, a facet of $K$ is a maximal simplex which does not serve as a face of any larger simplex.
With these definitions in mind, in the following theorem, we showcase the architecture of three-layer neural networks that can approximate $\indicator{\Xc}$ 
for a given simplicial complex $\Xc$.

\begin{theorem} \label{thm: simplicial complex}
    Let $\Xc\subset \Rd^d$ be a simplicial $m$-complex consists of $k$ facets, and let $k_j$ be the number of $j$-dimensional facets of $\Xc$. Then, for a given $\varepsilon>0$, there exists a three-layer neural network $\Nc$ with the architecture $\ReLUmax{d}{d_1}{k}$ such that $\Nc(\Rd^d)=[0,1]$, $\Nc(\xb)=1$, for $\xb\in\Xc$, and $\Nc(\xb)=0$ for $\xb\not\in B_\varepsilon(\Xc)$. 
    Furthermore, $d_1$ is bounded by
    \begin{align} \label{eq: simplicial complex} \hspace{-0.5cm} 
        d_1 \le 
        \min\left\{ k(d+1) - (d-1)\floor{ \sum_{j<\frac{d}{2}} \frac{k_j}{2}\!}
        ,\;
        (d+1)\! \left[\sum_{j\le \frac{d}{2}} \left(  k_j \frac{j+2}{d-j} + \frac{j+2}{j+1} \right) + \sum_{j>\frac{d}{2}} k_j \right] \right\}. 
    \end{align}
\end{theorem}
\begin{proof}[Proof Sketch.]
    Let $X_1, X_2, \cdots, X_k$ be $k$ facets of $\Xc$. {For each facet $X_i$, consider a $d$-simplex cover appeared in Lemma \ref{lem: simplex}. 
    Then Proposition \ref{prop: compact} provides the neural network $\Nc$ with the architecture $\ReLUmax{d}{(d+1)k}{k}$.}
    Lastly, removing the neurons in the first layer that can be overlapped, we get a slightly improved bound.
    The full proof can be found in Appendix \ref{app: proofs}.
\end{proof}

Theorem \ref{thm: simplicial complex} reveals that the width is restricted by the dimension $m$ and the number of facets $k$ of a provided simplicial complex. Looking at this from a topological perspective, it is generally intuitive that a smaller number of facets signifies a simpler structure of the simplicial complex. This notion is mathematically expressed in \eqref{eq: simplicial complex}, which suggests that when $m<\frac{d}{2}$ is fixed, the first maximum value in \eqref{eq: simplicial complex} results in $d_1 \lesssim k(d+1)-(d-1)\frac{k}{2} = \frac{k}{2}(d+3)$, which magnifies as $k$ increases.
Similarly, if $k$ is fixed and $\Xc$ consists of $m$-simplicies with $m<\frac{d}{2}$, the summation in the second maximum value in \eqref{eq: simplicial complex} reduces to $d_1 \lesssim (d+1) \left( k\frac{m+2}{d-m} + 2 \right)$, which rapidly diminishes as $m$ decreases. This suggests that a smaller dimension $m$ demands smaller widths, which aligns with the intuition that a low-dimensional manifold could be approximated with fewer neurons. To our knowledge, this is the first upper bound of the width of neural networks that depends on the topological data structure.

Theorem \ref{thm: simplicial complex} is a `universal' approximation result since the width bounds presented in \eqref{eq: simplicial complex} apply to any simplicial complex. However, being a general upper bound, there might be a smaller network architecture that can approximate a given simplicial complex. Proposition \ref{prop: compact} and Theorem \ref{thm: compact} show that the width upper bounds are determined by a convex polytope covering, which is heavily dependent on the topological features of $\Xc$.
In the upcoming toy example, we clarify how the upper bounds of width can vary based on the choice of covering, for the same topological space $\Xc$.

\begin{figure}[t]
    \centering
    \begin{subfigure}[b]{0.22\textwidth}
        \centering
        \includegraphics[width=\textwidth]{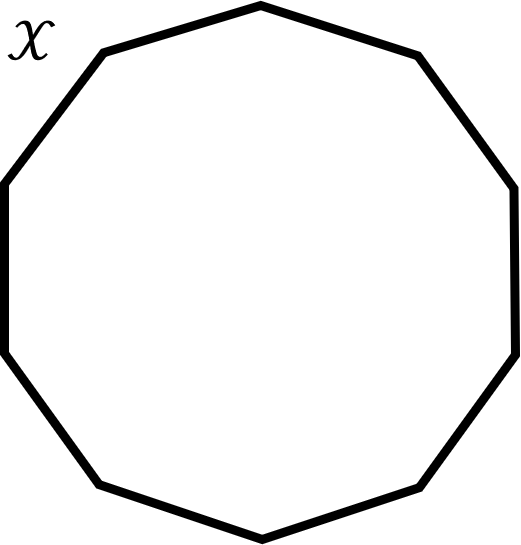}
        \caption{}
    \end{subfigure}
    \hfill
    \begin{subfigure}[b]{0.27\textwidth}
        \centering
        \includegraphics[width=\textwidth]{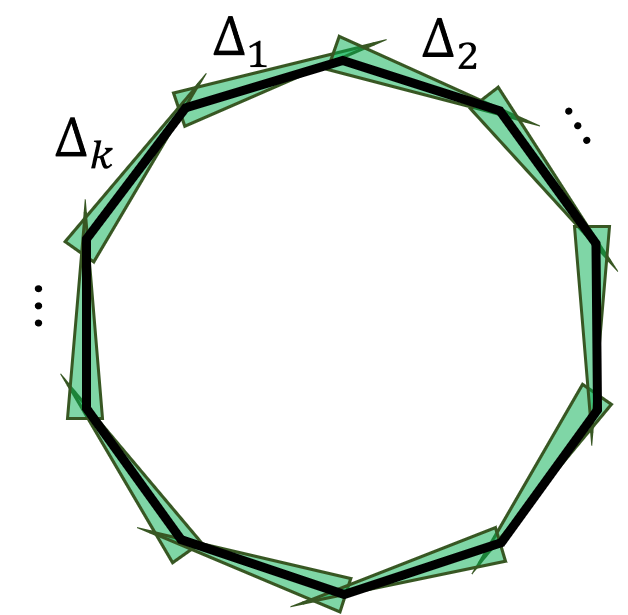}
        \caption{}
    \end{subfigure}
    \hfill
    \begin{subfigure}[b]{0.24\textwidth}
        \centering
        \includegraphics[width=\textwidth]{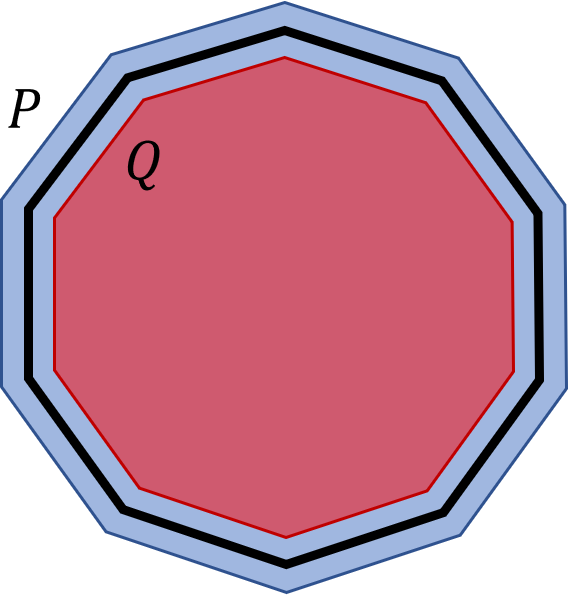}
        \caption{}
    \end{subfigure}
    \caption{A comparison of Proposition \ref{prop: compact}, Theorem \ref{thm: compact}, and Theorem \ref{thm: simplicial complex} in Example \ref{exp: k gon}.
    (a) Let $\Xc$ be the boundary of a regular $k$-gon in $\Rd^2$.
    (b) A triangle covering of $\Xc$ is given. Proposition \ref{prop: compact} provides the network architecture $\ReLUmax{2}{3k}{k}$.
    (c) $B_\varepsilon(\Xc)$ can be covered by the difference of two $k$-gon $P$ and $Q$. Then, Theorem \ref{thm: compact} guarantees that a network with the architecture  $\ReLUFour{2}{2k}{2}{2}{1}$ can approximate $\indicator{\Xc}$. 
    }
    \label{fig: polygon} 
\end{figure}
\begin{example}[Comparison of Proposition \ref{prop: compact}, Theorem \ref{thm: compact}, and Theorem \ref{thm: simplicial complex}] \label{exp: k gon}
    Let $\Xc$ be the boundary of a regular $k$-gon in $\Rd^2$ as shown in Figure \ref{fig: polygon}(a). 
    First, consider a convex polytope covering $\Cc$ that consists of $k$ green triangles in Figure \ref{fig: polygon}(b). 
    Proposition \ref{prop: compact} shows that a three layer network with the architecture $\ReLUmax{2}{3k}{k}$ can approximate $\indicator{\Xc}$.
    Since $\Xc$ can be regarded as a simplicial $2$-complex with $k$ facets, Theorem \ref{thm: simplicial complex} provides a slightly better bound $\ReLUmax{2}{\left(3k - \floor{\frac12 k}\right)}{k}$, but both networks still have $O(k)$ widths for two hidden layers.
    However, from its special structure of $\Xc$, it can be covered by difference of two solid $k$-gons $P$ and $Q$ (Figure \ref{fig: polygon}(c)). Then, Theorem \ref{thm: compact} provides a four-layer ReLU network with the architecture $\ReLUFour{2}{2k}{2}{2}{1}$, which has only one $O(k)$ width. 
    This example helps understand the benefits of depth from a topological perspective, akin to the advantages of depth studied in previous works \cite{bu2020depth, safran2017depth, telgarsky2016benefits, vardi2022width, wang2022approximation}. It is also important to note that the high number of neurons in the first layer is inevitable, as described in \cite{guss2018characterizing} and \cite{yang2021dive}.
\end{example} 

It is remarkable that our findings can be extended to various other neural network architectures. In Appendix \ref{app: extension}, we broaden our results to encompass deep ReLU networks (as seen in Corollary \ref{cor: four layer ReLU networks}) and networks using sigmoid activation function at last (as presented in Corollary \ref{cor: bce}).

\subsubsection{Upper bounds of widths in terms of Betti numbers}
The Betti number is a key metric used in TDA to denote the number of $k$-dimensional `holes' in a data distribution. Owing to its homotopy invariance, Betti numbers are frequently employed to study the topological features of a given topological space.
Interestingly, our prior results can be leveraged to ascertain a neural network architecture with width bounds defined in terms of the Betti numbers, given that the dataset $\Xc$ exhibits certain structural characteristics.

Specifically, Theorem \ref{thm: compact} offers upper width bounds when $\Xc$ can be depicted as a difference between unions of convex sets. Consequently, if $\Xc$ is convex and only contains `convex-shaped holes', we can derive a network width bound in relation to its Betti numbers. This concept is further explicated in the following theorem and example.

\begin{theorem} \label{thm: betti numbers}
    Suppose $\Xc\subset\Rd^d$ is a topological space obtained by removing some disjoint rectangular prisms from a $d$-dimensional cuboid. 
    Let $\beta_0, \beta_1, \cdots, \beta_d$ be the Betti numbers of $\Xc$. 
    Then for any $\varepsilon>0$, there exists a four-layer ReLU network $\Nc$ with the architecture 
    \begin{align} \label{eq: betti numbers}
        \ReLUFour{d\;} {\;2\left(d-1+\sum_{k=0}^d (k+1)\beta_k\right)\;} {\;\left(\sum_{k=0}^d \beta_k\right)} {2} {1}
    \end{align}
    such that $\Nc(\Rd^d)=[0,1]$, $\Nc(\xb)=1\;$ for $\xb\in\Xc$, and $\Nc(\xb)=0\;$ for $\xb\not\in B_\varepsilon(\Xc)$.
\end{theorem}
\begin{proof}[Proof Sketch] 
{
    Since each $k$-dimensional hole is a rectangular prism, we can consider it as convex polytopes. 
    The result is deduced from Theorem \ref{thm: compact} by computing the required number of faces of polytopes.}
    The full proof can be found in Appendix \ref{app: proofs}. 
\end{proof}

\begin{figure}[t]
    \centering
    \begin{subfigure}[b]{0.33\textwidth}
        \centering
        \includegraphics[width=\textwidth]{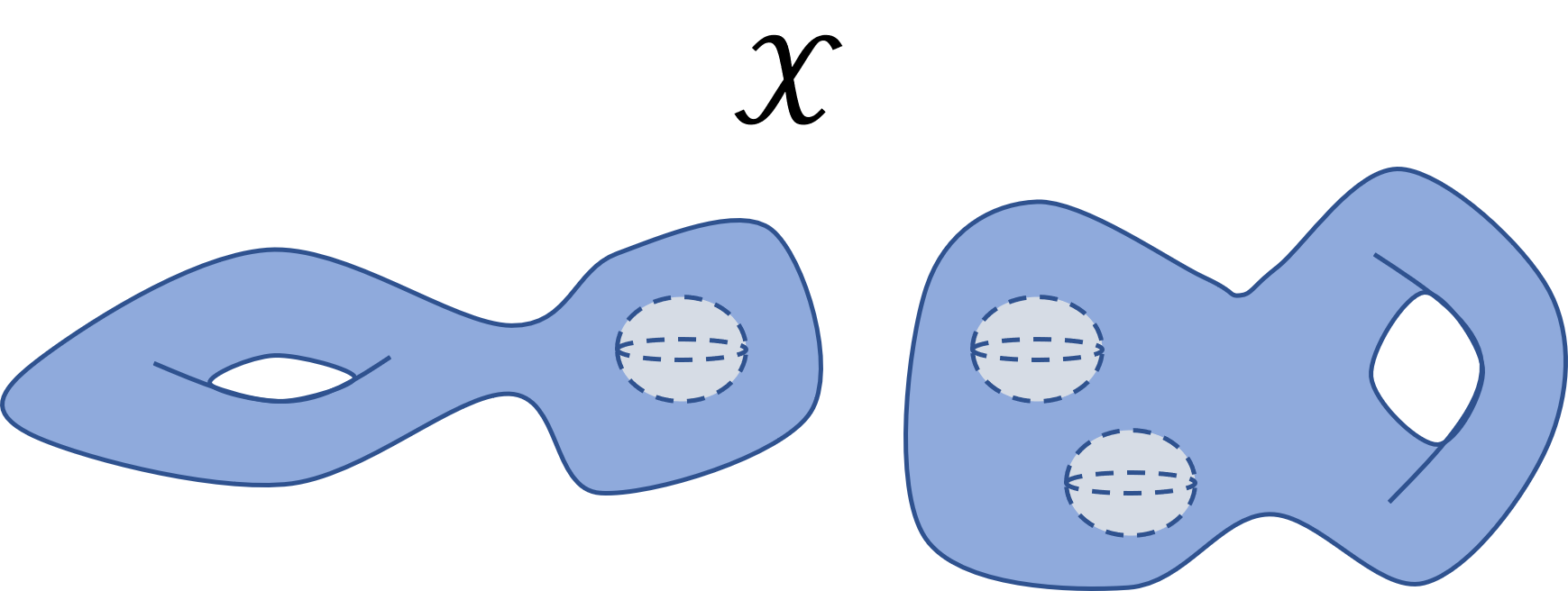}
        \caption{}
    \end{subfigure}
    \hfill
    \begin{subfigure}[b]{0.30\textwidth}
        \centering
        \includegraphics[width=\textwidth]{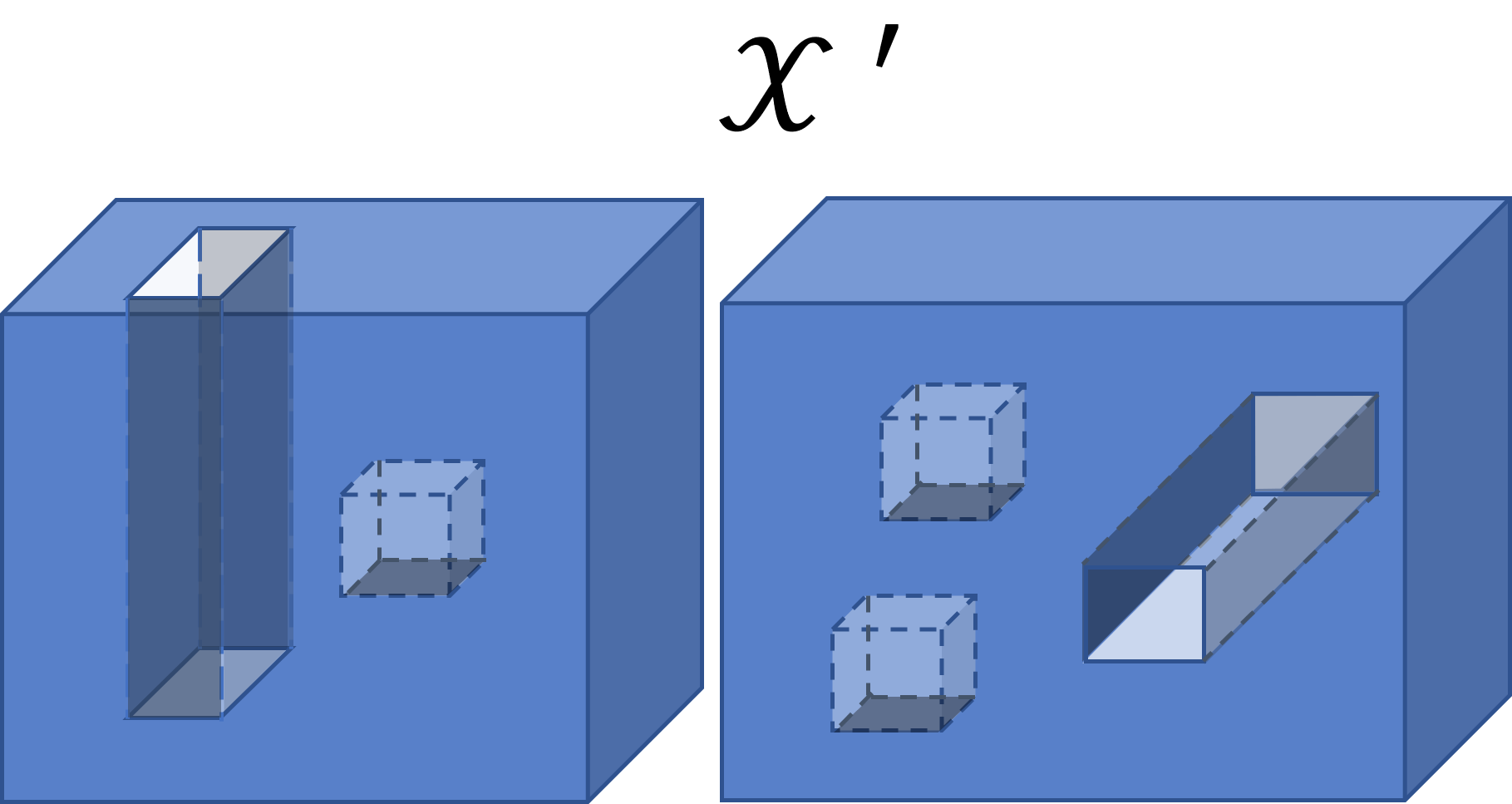}
        \caption{}
    \end{subfigure}
    \hfill
    \begin{subfigure}[b]{0.33\textwidth}
        \centering
        \includegraphics[width=\textwidth]{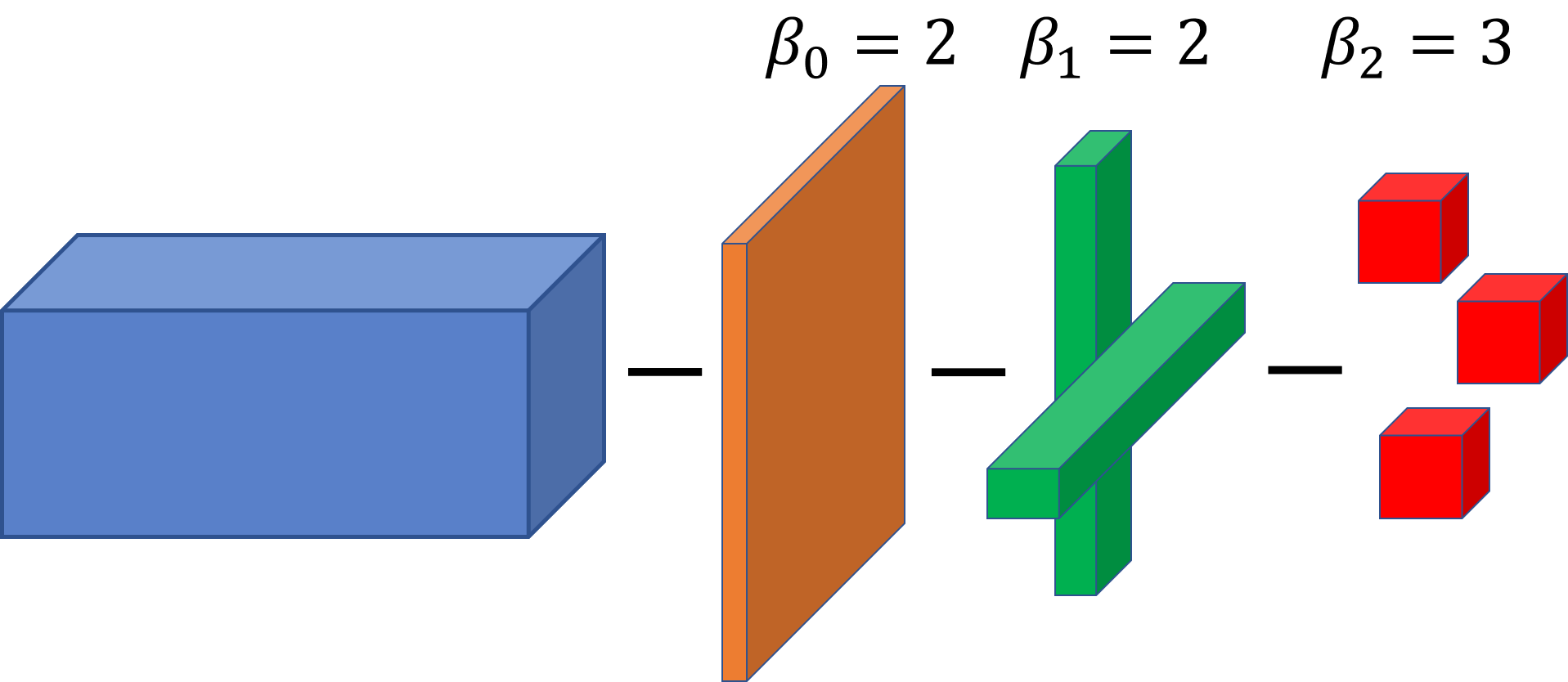}
        \caption{}
    \end{subfigure}
    \caption{An illustration of the correlation between Betti numbers and the network architecture (Example \ref{exp: betti numbers}).
    (a) A topological space $\Xc \subset \Rd^3$ is given, whose Betti numbers are $\beta_0=\beta_1=2$ and $\beta_2=3$.
    (b) A homotopy equivalent space $\Xc'$ is presented, obtained by removing several small cuboids from a larger one.
    (c) The removed cuboids from $\Xc'$ are shown.
    Theorem \ref{thm: betti numbers} demonstrates that a four-layer ReLU network with the architecture $\ReLUFour{3}{34}{7}{2}{1}$ can approximate $\indicator{\Xc'}$. \vspace{-0.6cm}
    }
    \label{fig: betti numbers}
\end{figure}
\begin{example} \label{exp: betti numbers}
    Let $\Xc$ be a topological space in $\Rd^3$ shown in Figure \ref{fig: betti numbers}(a),
    which has three nonzero Betti numbers $\beta_0=\beta_1=2$ and $\beta_2=3$. 
    Then we can consider the homotopy equivalent topological space $\Xc'\subset\Rd^3$ that satisfies the assumptions in Theorem \ref{thm: betti numbers}: 
    $\Xc'$ is obtained by `cutting out' a plate (orange), `punching' two rectangular prisms (green), and `hollowing out' three small cubes (red) from a large cuboid (blue) as described in Figure \ref{fig: betti numbers}(c).
    Then, Theorem \ref{thm: betti numbers} shows that a four-layer ReLU network with the architecture $\ReLUFour{3}{34}{7}{2}{1}$ can approximate $\indicator{\Xc'}$ arbitrarily closely.
\end{example}

Theorem \ref{thm: betti numbers} and Example \ref{exp: betti numbers} illustrate how the Betti numbers of a topological space $\Xc$ can contribute to defining upper bounds for the widths of neural networks. However, it's important to note that two homotopy equivalent spaces may require differing network architectures. 
In Proposition \ref{prop: non convex}, we prove that the indicator function over a crown-shaped topological space (Figure \ref{fig: non convex}(a)) in $\Rd^2$ cannot be approximated by a two-layer ReLU network with the architecture $\ReLUtwo{2}{3}{1}$, while a triangle can be approximated by Proposition \ref{prop: compact}. 
Since these two spaces have the same Betti numbers $\beta_0=1$, this example suggests that a neural network architecture cannot be solely determined by Betti numbers.

Nevertheless, Theorem \ref{thm: betti numbers} provides an upper bound on the widths of four-layer ReLU networks in terms of Betti numbers of $\Xc$, under the conditions stipulated. This is another novel result linking the topological characteristics of a dataset with upper bounds on widths. It is worth noting that similar results can be achieved when the cuboid assumptions in Theorem \ref{thm: betti numbers} are modified to other convex polytopes, using the same proof strategy.

We further elaborate on the topic of network architecture. Intriguingly, the sum of Betti numbers $\sum_{k=0}^d \beta_k$ that appears in the third layer in \eqref{eq: betti numbers} is termed the \emph{topological complexity} of $\Xc$. This quantity is recognized as a measure of the complexity of a given topological space \cite{naitzat2020topology}. This value has connections with other fields, for example, it has some lower and upper bounds from Morse theory \cite{milnor1963morse} and Gromov's Betti number Theorem \cite{gromov1981curvature}.
In the context of topological data analysis, consider a \u{C}ech complex constructed on a dataset $\Xc$ consisting of $n$ points, using a filtration parameter $\varepsilon$. Its topological complexity fluctuates from $n$ (when $\varepsilon=0$) to $1$ (when $\varepsilon>\text{diam}(\Xc)$) as $\varepsilon$ increases. This implies that the architecture in Theorem \ref{thm: betti numbers} is dictated by the filtration number $\varepsilon$, which controls the topological structure of the given dataset.
We believe this approach could inspire novel investigative methods in TDA, which we propose as an avenue for future research.

\section{Universal Approximation Property of Three-Layer ReLU Networks} \label{sec: regression}

In the preceding section, we demonstrated how indicator functions over certain topological spaces can be approximated by three-layer neural networks. Interestingly, this topological result has an application in proving the Universal Approximation Property (UAP) of three-layer ReLU networks. {Moreover}, we can derive upper bounds on the widths in three-layer ReLU networks for approximating Lipschitz functions. We present this result in the upcoming theorem.

\begin{theorem} \label{thm: regression}
    Let $d_x,d_y \in \Nd$ and $p\ge1$.
    Then, the set of three-layer ReLU networks is dense in $L^p(\Rd^{d_x}, [0,1]^{d_y})$.
    Furthermore, let $f : \Rd^{d_x} \rightarrow [0,1]^{d_y}$ be a compactly supported Lipschitz function.
    Then for any $\varepsilon>0$, there exists a three-layer ReLU network $\Nc$ with the architecture 
    \begin{align*}
        d_x \;\stackrel{\sigma}{\rightarrow}\;\ReLUtwo{(2n d_x d_y)\;}{\;(n d_y)\;}{d_y}
    \end{align*}
    such that $\norm{\Nc - f}_{L^p(\Rd^{d_x})}<\varepsilon$. 
    Here, $n= O(\varepsilon^{-d_x})$.
\end{theorem}
\begin{proof}[Proof Sketch] 
 The first assertion is a consequence of the second one. Regarding the second assertion, note that $f^*$, being Lipschitz, is continuous and Riemann integrable. Consequently, we can construct a linear combination of indicator functions that approximates $f^*${, which are known as simple functions in Lebesgue theory \cite{rudin1976principles}}. Each of these indicator functions can be implemented by a three-layer ReLU network using Proposition \ref{prop: compact}. The complete proof can be found in Appendix \ref{app: proofs}.
\end{proof}

Theorem \ref{thm: regression} makes two assertions. The first one affirms that the set of three-layer ReLU networks is dense in $L^p(\Rd^{d_x}, [0,1]^{d_y})$, aligning with the findings presented in \cite{wang2022approximation}. Considering that we have also demonstrated that the set of two-layer ReLU networks cannot universally approximate a compactly supported function (as per Proposition \ref{prop: two-layer cannot}), we can conclude that the minimum depth of DNNs to achieve universal approximation in $L^p(\Rd^{d_x})$ is precisely 3. This conclusion refines the results shown in \cite{wang2022approximation}.
Moreover, while \cite{wang2022approximation} merely demonstrates the possibility of approximating functions using three-layer ReLU networks, the second assertion of Theorem \ref{thm: regression} provides upper bounds on the widths, which are $O(\varepsilon^{-d_x})$, when the function to be approximated, $f^*$, is Lipschitz and compactly supported. As far as we are aware, this is the first work to present an upper bound on the width of three-layer ReLU networks for UAP.

These findings open a potential path to extend our topological results. We expect that for certain classes of functions, the width bounds could be reduced by adding more layers to the network, as demonstrated in Theorem \ref{thm: compact} and Example \ref{exp: k gon}. However, we leave this as an area for future research.

\section{Experimental Results} \label{sec: experiment} 

In Section \ref{sec: main}, we introduced a construction for three or four-layer neural networks that can approximate $\indicator{\Xc}$ for any given topological space $\Xc$ with a sufficient degree of accuracy. Naturally, this leads us to the question of whether these networks can be produced using a gradient method. While most prior theoretical studies concerning the existence of neural networks have not included experimental verification \cite{huang2020relu, huang2022theoretical, ismailov2020three, park2020minimum}, the issue of experimental verification remains an important one, as noted in \cite{vardi2022width}. In this section, we will provide experimental evidence that gradient descent can indeed converge to the neural networks that we constructed in Section \ref{sec: main}.

We consider two illustrative manifolds, $\Xc_1$ and $\Xc_2$, depicted in Figure \ref{fig: experiments}(a) and (d) respectively. The first compact set $\Xc_1$ is a simplicial $2$-complex in $\Rd^2$ comprised of two triangles. The second compact set $\Xc_2$ is a hexagon with a pentagonal hole, which is not a simplicial complex but a compact manifold. The datasets are constructed by selecting 1600 equidistant lattice points in the domain $[-20,20]\times[-20,20] \subset \Rd^2${, where each point has} label `$1$' if it lies on the manifold $\Xc_i$, and `$0$' otherwise.
We undertake both regression and classification tasks, using mean square error (MSE) loss and binary cross entropy (BCE) loss respectively. For the BCE loss task, we adhere to the architecture proposed in Corollary \ref{cor: bce} to ensure trainability. We employ the gradient descent algorithm for training our networks. For a clearer visualization of weight vectors in each layer, we plot the lines of vanishing points for each layer in blue (1st layer), red (2nd layer), etc. The grayscale color denotes the output range of the trained network.

\begin{figure}[t]
    \centering
    \begin{subfigure}[b]{0.3\textwidth}
        \centering
        \includegraphics[width=\textwidth]{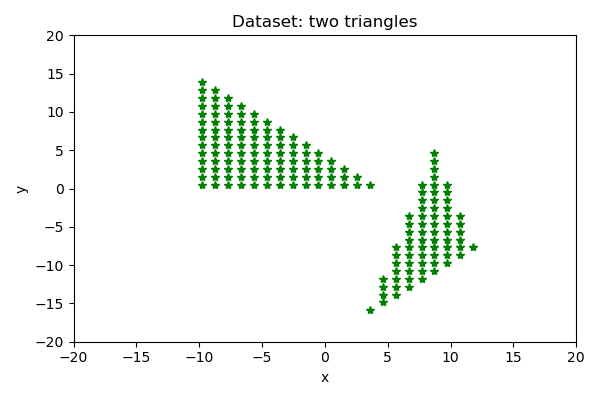}
        \caption{}
    \end{subfigure}
    \hfill
    \begin{subfigure}[b]{0.33\textwidth}
        \centering
        \includegraphics[width=\textwidth]{TT_MSE.png}
        \caption{}
    \end{subfigure}
    \hfill
    \begin{subfigure}[b]{0.33\textwidth}
        \centering
        \includegraphics[width=\textwidth]{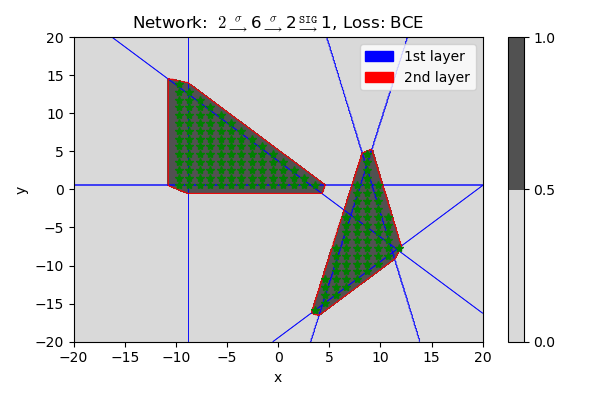}
        \caption{}
    \end{subfigure}
    \\
    \begin{subfigure}[b]{0.3\textwidth}
        \centering
        \includegraphics[width=\textwidth]{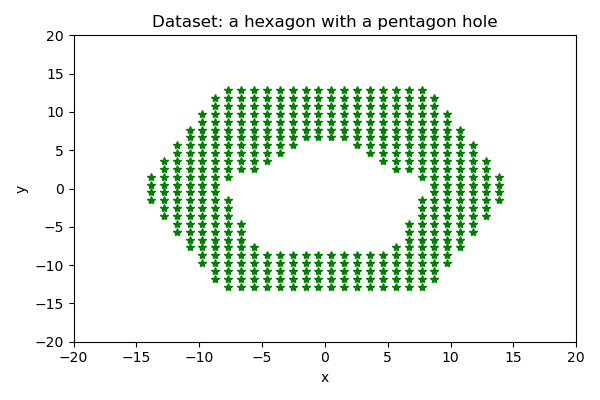}
        \caption{}
    \end{subfigure}
    \hfill
    \begin{subfigure}[b]{0.33\textwidth}
        \centering
        \includegraphics[width=\textwidth]{HEX_MSE.png}
        \caption{}
    \end{subfigure}
    \hfill
    \begin{subfigure}[b]{0.33\textwidth}
        \centering
        \includegraphics[width=\textwidth]{HEX_BCE.png}
        \caption{}
    \end{subfigure}
    \caption{Experimental verification of convergence of gradient descent.
    (a) and (d) exhibit the shape of two data manifolds, which are `two triangles' and `a hexagon with a pentagon hole'. 
    (b) and (e) show the converged networks by gradient descent, when the loss function is given by the mean square error (MSE) loss.
    Similarly, (e) and (f) show the results for the binary cross entropy (BCE) loss. 
    These results verify that gradient descent indeed converges to the networks we proposed in Section \ref{sec: main}.
    }
    \label{fig: experiments}
\end{figure}

For the first data manifold $\Xc_1$, Theorem \ref{thm: simplicial complex} suggests that a network with the architecture $\ReLUmax{2}{6}{2}$ can precisely represent $\indicator{\Xc_1}$. As depicted in Figure \ref{fig: experiments}(b), the trained network indeed converges exactly to the network proposed in the theorem under MSE loss. The weight vectors in the first layer encapsulate the two triangles, reflecting the topology of $\Xc_1$. A similar result can be observed for a network with the architecture $2\stackrel{\sigma}{\rightarrow} 6 \stackrel{\sigma}{\rightarrow} 2 \stackrel{\texttt{SIG}}{\rightarrow} 1$, as suggested by Corollary \ref{cor: bce}.
For the second data manifold $\Xc_2$, Theorem \ref{thm: compact} suggests that a network with the architecture $\ReLUFour{2}{11}{2}{2}{1}$ can fit this manifold, {and} Figure \ref{fig: experiments}(e) verifies this under MSE loss. The architecture $2\stackrel{\sigma}{\rightarrow} 11 \stackrel{\sigma}{\rightarrow} 2\stackrel{\texttt{SIG}}{\rightarrow} 1$ also converges to the network proposed in Corollary \ref{cor: bce}, as shown in Figure \ref{fig: experiments}(f). More specifically, {the eleven weight vectors in the first layer} align with the {eleven} boundaries of the outer hexagon and the inner pentagon, with two neurons in the second layer encapsulating each polygon.

These experimental results provide evidence that the networks proposed in Section \ref{sec: main} can indeed be reached as {the global} minima by gradient descent, under both MSE loss and BCE loss.
However, we must stress the importance of initialization. Although we observe successful convergence to the expected networks, the results are heavily dependent on the initialization. For instance, under random initialization, it is known that all neurons in a layer may `die' at initialization with a nonzero probability \cite{lu2019dying}, leading to a poor network performance. {We provide further experimental results concerning different initializations in Appendix \ref{app: experiment}.}
The key takeaway from this section is that \emph{the networks proposed in Section \ref{sec: main} can indeed be reached by gradient descent}, a fact not demonstrated by prior studies on UAP, even in the context of toy examples.

\section{Conclusions} \label{sec: conclusion}
While many previous studies have explored the universal approximation property of DNNs, they have largely overlooked the connection with the topological features of datasets. In this paper, we have addressed this gap by providing data topology-dependent upper bounds on the width of DNNs. 
We have shown that for a dataset $\Xc$ that can be tightly covered by convex polytopes, a three-layer network architecture can be derived to approximate the indicator function $\indicator{\Xc}$. We also extended this to four-layer ReLU networks for $\Xc$ that can be covered by the difference of two unions of convex polytopes. 
We further generalized this to simplicial $m$-complexes and demonstrated the construction of a three-layer neural network with ReLU activation and max pooling operation. Imposing further assumptions on $\Xc$, we proposed the architecture of four-layer ReLU networks whose widths are bounded in terms of Betti numbers, a novel result in this field. 
Finally, we demonstrated that to approximate a compactly supported {Lipschitz} function in $\Rd^{d_x}$ by a three-layer ReLU network under an $\varepsilon$ error, a width of $O(\varepsilon^{-d_x})$ is sufficient. 
Through our experiments, we showed that gradient descent indeed converges to the networks we proposed.

\paragraph{Limitations and Future Works}
There are several limitations. Firstly, although we experimentally showed convergence to the constructed network, we did not provide a rigorous guarantee for this convergence {here}. To overcome this and to ensure both existence and convergence, it would be valuable to prove a convergence result under specific conditions.
Secondly, further research is needed to investigate the relationship between the network architecture and Betti numbers of the given topological spaces. The results in Theorem \ref{thm: betti numbers} suggest that it may be possible to expand our findings or reduce some assumptions on the given topological space. Since Betti numbers are commonly used features to characterize topological features of data manifolds in Topological Data Analysis (TDA), we believe that this line of research could yield significant insights.

\newpage
\bibliographystyle{plain}

\newpage
\appendix
{\Large \textbf{Appendix}}

\section{Extension to Deep ReLU Networks and Cross Entropy Loss} \label{app: extension}
In Section \ref{sec: main}, we presented a variety of three and four-layer neural networks utilizing ReLU activation and max pooling operations. In this section, due to its unique structure, we extend our discussion to other neural network architectures. Particularly, we focus on networks that exclusively use ReLU activations, which are prevalent and effective in practical applications.

Interestingly, the principle outlined in Lemma \ref{lem: convex polytope} enables us to replace the $\texttt{MAX}$ operation in the final layer with an additional ReLU layer. This revelation lays the groundwork for the following corollary, which directly stems from Proposition \ref{prop: compact} and Theorem \ref{thm: simplicial complex}

\begin{corollary}[Four-layer ReLU networks] \label{cor: four layer ReLU networks}
    The three-layer neural network in Proposition \ref{prop: compact} or Theorem \ref{thm: simplicial complex} can be changed to a four-layer ReLU network with the architecture $\ReLUThree{d}{d_1}{d_2}{1}\rightarrow1$, where $d_1$ and $d_2$ are exhibited in the proposition or theorem.
\end{corollary}
\begin{proof}
    Consider the three-layer neural network proposed in Proposition \ref{prop: compact} or Theorem \ref{thm: simplicial complex} which has the architecture $\ReLUmax{d}{d_1}{d_2}$.
    Let $a_1, \cdots, a_{d_2}$ be the input of the last layer, thus the output of the second layer.
    Since $a_1, \cdots, a_{d_2} \in [0,1]$ from the construction, we get $\texttt{MAX}(a_1, \cdots, a_k) = 1 - \sigma(1-a_1-\cdots-a_k)$.
    This completes the proof.
\end{proof}

We now turn our attention to neural networks trained under binary cross entropy (BCE) loss. For a single pair of data $\xb$ and its corresponding label $y$, the BCE loss is defined as $\ell(\xb,y) := y\log(\Nc(\xb))+(1-y)\log(1-\Nc(\xb))$. Hence, during training, the output of the neural network must neither be zero nor exceed $1$. This requirement is the primary reason classifiers utilize the sigmoid activation function $\texttt{SIG}(x) := \frac{1}{1+e^{-x}}$.

In light of this, we extend our findings to accommodate a network that employs sigmoid activation in the final layer. This adjustment can be readily achieved using our previous results, as detailed in Corollary \ref{cor: four layer ReLU networks}. Notably, this expansion does not necessitate extra layers, unlike in Corollary \ref{cor: four layer ReLU networks}. The outcome is presented in the subsequent corollary.

\begin{corollary}[Cross entropy loss] \label{cor: bce}
    Let $\Nc$ be the neural network proposed in one of Proposition \ref{prop: compact}, Theorem \ref{thm: compact}, Theorem \ref{thm: simplicial complex}, or Theorem  \ref{thm: betti numbers}. Then, the last activation function of $\Nc$ can be replaced by $\texttt{SIG}$, with $\Nc$ satisfying
    \begin{align*}
        \Nc(\xb) &> 1-\delta                && \text{if } \xb \in \Xc, \\ 
        \Nc(\xb) &< \delta                && \text{if } \xb\not\in B_\varepsilon(\Xc)
    \end{align*}
    for any given $\delta>0$.
\end{corollary}
\begin{proof}
    The proof is similar with the proof of Corollary \ref{cor: four layer ReLU networks}.
    If the last layer has $\texttt{MAX}$ activation, then for the inputs $a_1,a_2, \cdots, a_k$, replace $\texttt{MAX}(a_1, \cdots, a_k)$ to $\texttt{SIG}(M(-1+a_1+a_2+\cdots+a_k))$ with sufficiently large $M>0$. 
    If the last layer has ReLU activation, then for the inputs $a$ and $b$, replace $\sigma(a - b)$ to $\texttt{SIG}(M(a-b-\frac12))$ with sufficiently large $M>0$.
    It is easy to verify that these substitutions satisfy the desired property. 
\end{proof}

\section{Proofs of Propositions and Theorems} \label{app: proofs}

\begin{proof}[Proof of Proposition \ref{prop: two-layer cannot}]
    We prove by contradiction: suppose two-layer ReLU networks are dense in $L^p (\Rd^d)$. Then for a nonzero function $f^*$ and $\varepsilon>0$, there exists a nonzero two-layer ReLU network $f=\sum_{i=1}^k  v_i\sigma(\wb_i^\top\xb+b_i) + b_0$ such that $\norm{f-f^*}_{L^p(\Rd^d)} < \varepsilon$. Without loss of generality, we can assume that all $v_i, \wb_i$ are nonzero, i.e., $f$ has the minimal representation. Now, let $K$ be a compact (thus bounded) set that contains the support of $f^*$. Since $f$ is piecewise linear and defined on the unbounded domain $\Rd^d$, we can choose an unbounded partition $A\subset\Rd^d$ such that $f$ is linear in $A$ and $A \cap K^c$ is unbounded.
    By re-ordering of indices if needed, there exists $k_A\in\Nd$ such that  $\wb_i^\top\xb+b_i \ge 0$ if and only if $i \in [k_A]$.
    Then for ${\xb\in A}$, $f(\xb) = \sum_{i=1}^{k_A}v_i(\wb_i^\top\xb+b_i) + b_0$ and we get 
    \begin{align*}
        \norm{f-f^*}_{L^p(\Rd^d)}^p &= \int_{\Rd^d} |f-f^*|^p d\mu \\
        &\ge \int_{A\cap K^c} |f-f^*|^p d\mu \\
        &\ge \left(\inf_{\xb \in A\cap K^c}{|f(\xb)|^p}\right) \cdot \mu(A\cap K^c)
    \end{align*}
    where $\mu(\cdot)$ denotes the Lebesgue measure on $\Rd^d$. Since $K$ is compact and $A$ is unbounded, $\mu(A\cap K^c)=\infty$ conclude that $\inf\limits_{\xb \in A\cap K^c} |f(\xb)| = 0$ for $\xb \in A\cap K^c$.  Since $f(\xb) = \sum\limits_{i=1}^{k_A} v_i(\wb_i^\top\xb+b_i) + b_0$ is linear on $A$, we have 
    \begin{align} \label{eq: partition A}
        \sum_{i=1}^{k_A}v_i\wb_i = \zerob \qquad\text{and}\qquad  
        b_0+ \sum_{i=1}^{k_A}v_ib_i = 0.
    \end{align}
    Now consider the adjacent unbounded partition $B$ such that $f$ is linear in $B$ and $B\cap K^c$ is also unbounded (such partition $B$ can be chosen by a linear partition of $f$ in $A^c \cap K^c$). Through the exactly same arguments, we obtain the similar conclusion with \eqref{eq: partition A} on $B$.
    \begin{align} \label{eq: partition B}
        \sum_{i=1}^{k_B}v_i\wb_i=0 \qquad\text{and}\qquad  
        b_0 + \sum_{i=1}^{k_B}v_i b_i = 0.
    \end{align}
    However, since $B$ is the adjacent partition of $A$, exactly one neuron (call $v', \wb'$) is either activated or deactivated in $B$. Comparing \eqref{eq: partition A} and \eqref{eq: partition B}, we get either $v'=0$ or $\wb'=\zerob$, which contradicts to the minimality assumption. This completes the proof.
\end{proof}

\vspace{0.5cm}
\begin{proof}[Proof of Theorem \ref{thm: simplicial complex}]
    Let $X_1, X_2, \cdots, X_k$ be the $k$ facets of $\Xc$. For each facet $X_i$, we can construct a two-layer ReLU network $\Tc_i$ such that $\Tc_i(\xb)=1$ for $\xb\in X_i$ and $\Tc_i(\xb)<0$ $\xb\not\in B_\varepsilon(X_i)$ by Lemma \ref{lem: simplex}.
    Then Proposition \ref{prop: compact} gives a neural network $\Nc$ with the architecture $\ReLUmax{d}{d_1}{k}$ with $d_1 \le k(d+1)$, such that $\Nc$ can approximate $\indicator{\Xc}$ arbitrarily close.
    The remaining goal is to reduce the width of the first layer. 
    
    From the construction, recall that $d_1\le k(d+1)$ comes from the fact that each simplex $X_i$ is covered by a $d$-simplex which has $(d+1)$ hyperplanes. 
    Now consider two $m$-simplices in $\Rd^d$. If $2m+2 \le d+1$, then we can connect all points of the two $m$-simplices in $\Rd^d$, and it becomes a $(2m+2)$-simplex $\Delta^{2m+2}$. Now construct a $d$-simplex $\Delta^{d+1}$ by choosing $(d+1)-(2m+2)$ points in $B_\varepsilon(\Delta^{2m+2})$, whose base is this $(2m+2)$-simplex. Then, by adding two distinguishing hyperplanes at last, we totally consume $d+3$ hyperplanes to separate two $m$-simplices. 
    
    Now we apply this argument to each pair of two simplices. The above argument shows that two $m$-simplices separately covered by $2(d+1)$ hyperplanes can be re-covered by $(d+3)$ hyperplanes if $m\le\floor{\frac{d-1}{2}}$, which reduces $(d-1)$ number of hyperplanes.
    In other words, we can save $(d-1)$ hyperplanes for each pair of two $m$-simplices whenever $m\le\floor{\frac{d-1}{2}}$. This provides one improved upper bound of $d_1$:
    \begin{align} \label{eq: a bound of d1}
        d_1 \quad\le\quad k(d+1) - (d-1)\floor{\frac{1}{2} \sum_{j=0}^{\floor{\frac{d-1}{2}}} k_j}.
    \end{align}
    
    Now, we consider another pairing. For $0\le j \le m$, $\Xc$ has $k_j$ $j$-simplex facets. Recall that each $j$-simplex has $(j+1)$ points, and a $d$-simplex consists of $(d+1)$-points. Therefore, all points in $\floor{\frac{d+1}{j+1}}$ $j$-simplices can be contained in one $d$-simplex. In this case, these $j$-simplices can be covered by adding $\floor{\frac{d+1}{j+1}}$ hyperplanes more. Thus if we have $k_j$ $j$-simplices, then the required number of hyperplanes to separatedly encapsulate the $j$-simplices is less than or equal to
    \begin{align}
        \# \left(\text{the number of }d\text{-simplices}\right) &\cdot \# \left( \text{the required number of hyperplanes in each $d$-simplex} \right) \notag \\
        &=\left( \floor{\frac{k_j}{\floor{\frac{d+1}{j+1}}}}+1 \right) \cdot \left( d+1+\floor{\frac{d+1}{j+1}} \right) \notag \\
        &\le \left( k_j \frac{j+1}{d-j}+1\right) \cdot \left( d+1 + \frac{d+1}{j+1} \right) \notag \\
        &< (d+1) \left(\frac{j+2}{j+1}\right) \left( k_j \frac{j+1}{d-j}+1\right) \notag \\
        &= (d+1) \left( k_j \frac{j+2}{d-j} + \frac{j+2}{j+1} \right) \label{eq: however}
    \end{align}
    where the inequality is reduced from the property of the floor function: $ a-1 < \floor{a} \le a<\floor{a}+1$ for any $a\in\Rd$. Then another upper bound of $d_1$ is obtained by applying \eqref{eq: however} for all $j \le m$. However, also note that \eqref{eq: however} is greater than the known upper bound $k(d+1)$ if $j>\frac{d}{2}$ ; the sharing of covering simplex is impossible in this case. Therefore, the upper bound of $d_1$ is given by 
    \begin{align}
        d_1  \quad&\le (d+1) \sum_{j\le \frac{d}{2}} \left(  k_j \frac{j+2}{d-j} + \frac{j+2}{j+1} \right) + (d+1)\sum_{j>\frac{d}{2}} k_j \notag\\
        &=(d+1) \left[\sum_{j\le \frac{d}{2}} \left(  k_j \frac{j+2}{d-j} + \frac{j+2}{j+1} \right) + \sum_{j>\frac{d}{2}} k_j \right]
         \label{eq: another bound of d1}
    \end{align}
    
    Therefore, from \eqref{eq: a bound of d1} and \eqref{eq: another bound of d1}, we get
    \begin{align*}
        d_1 \le
        \min\left\{ k(d+1) - (d-1)\floor{\frac{1}{2} \sum_{j=0}^{\floor{\frac{d-1}{2}}} k_j} 
        ,\; 
        (d+1) \left[\sum_{j\le \frac{d}{2}} \left(  k_j \frac{j+2}{d-j} + \frac{j+2}{j+1} \right) + \sum_{j>\frac{d}{2}} k_j \right] \right\}.
    \end{align*}
\end{proof}

\begin{proof}[Proof of Theorem \ref{thm: betti numbers}]
    The proof is achieved by applying Theorem \ref{thm: compact}, where each removed rectangular prisms corresponds to a convex polytope $Q_j$. We proceed the proof by following the Betti numbers. Before we start, we set $P$ to the largest cuboid, which is enveloped by $2d$ hyperplanes.
    
    First, consider $\beta_0$, the number of connected components. To distinguish two connected components, we need $2$ hyperplanes for one cutting plane (Figure \ref{fig: betti numbers}(c)). This is directly generalized to $\beta_0$ connected components, which needs $\beta_0-1$ cutting planes. Therefore, the number of required hyperplanes is $2(\beta_0-1)$.
    
    Second, consider $\beta_1$. By the similar argument, these long rectangular prism has a rectangle base enveloped by $2\times 2$ additional hyperplanes, where other hyperplanes are shared from the large cuboid $P$. Then, the required number of hyperplanes is $2\cdot 2 \beta_1$.
    
    Now we generalize this argument. In the $d$-dimensional cuboid $P$, a $k$-dimensional hole (associated to the Betti number $\beta_k$) with $k\ge1$ is made by $2(k+1)$ additional hyperplanes and other hyperplanes of $P$.
    Since all `hole's are disjoint, each hole is independently created. Therefore, the total number of required hyperplane is
    \begin{align*}
        2d + 2(\beta_0-1) + 2\sum_{k=1}^d (k+1)\beta_k
        \quad=\quad 2\left(d - 1 + \sum_{k=0}^d (k+1)\beta_k \right).
    \end{align*}
    Clearly, from the construction, the number of removed polytopes is $\beta_0-1 + \sum_{k=1}^d \beta_k$, which corresponds to $n_Q$ in Theorem \ref{thm: compact}. Since $n_P=1$, we obtained the width of the third layer by $n_P+n_Q = \sum_{k=0}^d \beta_k$. 
    Therefore, Theorem \ref{thm: compact} guarantees that there exists a four-layer ReLU network with the architecture
    \begin{align*}
        \ReLUFour{2\;}{\;2\left(d - 1 + \sum_{k=0}^d (k+1)\beta_k \right)\;}{\;\left(\sum_{k=0}^d \beta_k\right)\;}{2}{1}
    \end{align*}
    that can approximate $\indicator{\Xc}$ arbitrarily close.
\end{proof}

\vspace{0.5cm}
\begin{proof}[Proof of Theorem \ref{thm: regression}]
    Fist we recall one result in real analysis: ``the set of compactly supported continuous functions is dense in $L^p(\Rd^{d_x})$ for $p \ge 1$ \cite[Theorem 3.14]{rudin1976principles}.''
    Therefore, it is enough to prove the second statement; ``any compactly supported Lipschitz function can be universally approximated by a three-layer ReLU network.''
    
    We consider $d_y=1$ case first. 
    Let $f: \Rd^{d_x} \rightarrow [0,1]$ be a Lipschitz function with the Lipschitz constant $L$. Without loss of generality, suppose the support of $f$ is contained in $[0,1]^{d_x}$.
    Let $\delta>0$ be a small unit fraction which will be determined later. We partition $[0,1]^{d_x}$ by regular $d_x$-dimensional cubes with length $\delta$. Now, consider estimating a definite integral by a Riemann sum over these cubes.
    The total number of cubes is $n:=(\frac1\delta)^{d_x}$, and we number the cube by $C_1, C_2, \cdots, C_n$. In each cube $C_i$, by Lemma \ref{lem: convex polytope}, we can define a two-layer ReLU network $\Tc_i$ with the architecture $\ReLUTwo{d_x}{2d_x}{1}$ such that $\Tc_i(\xb)=1$ for $\xb\in C_i$ and $\Tc_i(\xb)=0$ for $\xb \not\in B_{r}(C_i)$ with $r:=\frac{1}{2d_x} \frac{\delta^{p+1}}{1+\delta^p}$.
    Then for any $\xb_i \in C_i$, we get
    \begin{align*}
        \int_{B_{r}(C_i)} |f-f(\xb_i)\Tc_i|^p \; d\mu
        &= \int_{C_i} |f-f(\xb_i)\Tc_i|^p \; d\mu + \int_{B_{r}(C_i) \backslash C_i} |f-f(\xb_i)\Tc_i|^p \; d\mu \\
        &\le \int_{C_i} (\sqrt{d_x}L\delta)^p \; d\mu  + \int_{B_{r}(C_i) \backslash C_i} 1^p \; d\mu \\
        &\le (\sqrt{d_x}L\delta)^p \cdot \delta^{d_x} + \left[ (\delta+2r)^{d_x} - \delta^{d_x} \right] \\
        &=  (\sqrt{d_x}L)^p \cdot \delta^{d_x+p} + \left[\left(1+\frac{2r}{\delta}\right)^{d_x} -1\right] \delta^{d_x}  \\
        &< \left[ (\sqrt{d_x}L)^p + 1 \right] \delta^{d_x+p}. 
    \end{align*}
    Note that we use two inequalities, $|f(\xb) - f(\xb_i)| \le L\norm{\xb-\xb_i} \le L\sqrt{d_x}\delta$ for $\xb\in C_i$ and $(1+a)^k < \frac{1}{1-ak}$ for $0<a<\frac{1}{k}$. 
    Then, the above equation implies the $L^p$ distance between $f$ and $f(\xb_i)\Tc_i$ in $B_r(C_i)$ is bounded by the above value. 
    Now we define a three-layer neural network $\Nc$ by
    \begin{align*}
        \Nc(\xb) := \sum_{i=1}^n f(\xb_i) \Tc_i(\xb),
    \end{align*} 
    which is the Riemann sum over the $n$ cubes.
    Then $\Nc$ has the architecture ${d_x}\stackrel{\sigma}{\rightarrow}\ReLUtwo{(2n d_x)}{n}{1}$ and satisfies
    \begin{align*}
        \int_{\Rd^{d_x}} |f-\Nc|^p d\mu 
        &= \int_{B_r([0,1]^{d_x})} |f-\Nc|^p \; d\mu \\
        &< \sum_{i=1}^n \int_{B_r(C_i)} |f-f(\xb_i)\Tc_i|^p \; d\mu \\
        &\le \left[ (\sqrt{d_x}L)^p + 1 \right] n \delta^{d_x+p}. \\
        &= \left[ (\sqrt{d_x}L)^p + 1\right] \delta^p.
    \end{align*}
    Therefore, choose the unit fraction $\delta$ to satisfy $\delta <\varepsilon \left[1+(\sqrt{d_x}L)^p\right]^{-\frac1p}$ for the given $\varepsilon$, we conclude that $\norm{f-\Nc}_{L^p([0,1]^{d_x})} < \varepsilon$. 
    From this choice of $\delta$, we get
    \begin{align*}
        n = \delta^{-d_x} = O(\varepsilon^{-d_x}).
    \end{align*}
    Lastly, for $d_y>1$, we can obtain the desired network by concatenating $d_y$ networks. Therefore, the architecture of such neural network is 
    \begin{align*}
       {d_x}\stackrel{\sigma}{\rightarrow}\ReLUtwo{(2n d_x d_y)}{(n d_y)}{d_y}.
    \end{align*}
\end{proof}

\section{Additional Propositions and Lemmas}

\begin{proposition} \label{prop: epsilon}
    Let $\Xc$ be a compact set in $\Rd^d$. For arbitrary $\delta>0$ and $p>0$, suppose there exists a function $f_\delta:\Rd^d\rightarrow\Rd$ such that
    $f_\delta(\Rd^d)=[0,1]$ and 
    \begin{align*}
        f_\delta(\xb) &= 1 \qquad \text{if } \xb \in \Xc, \\
        f_\delta(\xb) &= 0 \qquad \text{if } \xb \not\in B_{\delta}(\Xc).
    \end{align*}
    Then, for arbitrary $\varepsilon>0$, there exists a function $\Nc: \Rd^d\rightarrow\Rd$ such that
    \begin{align} \label{eq: epsilon}
       \norm{\Nc(\xb) - \indicator{\Xc}(\xb)}_{L^p(\Rd^d)} < \varepsilon. 
    \end{align}
\end{proposition}
\begin{proof}
    Let $\mu$ be the Lebesgue measure in $\Rd^d$. 
    First note that 
    \begin{align*}
        \lim\limits_{\delta\rightarrow0^+} \mu(B_{\delta}(\Xc)-\Xc) = \mu(\bar{\Xc}\backslash\Xc) = 0.
    \end{align*}
    Therefore, for a given $\varepsilon$, there exists $\delta>0$ such that 
    \begin{align*}
        \mu\left( B_\delta(\Xc) - \Xc \right) < \varepsilon^p .
    \end{align*}
    From the assumption, for such $\delta$, there exists a function $f_\delta: \Rd^d\rightarrow\Rd$ that satisfies $f_\delta(\Rd^d)=[0,1]$,  $f_\delta(\xb)=1$ for $\xb\in\Xc$, and $f_\delta(\xb)=0$ if $\xb\not\in B_{\delta}(\Xc)$.
    Now, define $\Nc:=f_\delta$. Then, 
    \begin{align*}
        \norm{\Nc(\xb) - \indicator{\Xc}(\xb)}_{L^p(\Rd^d)}^p
        &= \int_{\Rd^d} |\Nc(\xb) - \indicator{\Xc}(\xb)|^p \;d\mu \\
        &= \int_{B_\delta(\Xc)} |\Nc(\xb) - \indicator{\Xc}(\xb)|^p \;d\mu \\
        &= \int_{B_\delta(\Xc)\backslash\Xc} |\Nc(\xb) - \indicator{\Xc}(\xb)|^p \;d\mu \\
        &\le 1^p \cdot \mu \left( {B_\delta(\Xc)\backslash\Xc} \right)\\
        &< \varepsilon^p.
    \end{align*}
    Therefore, $\Nc$ is the desired function satisfying \eqref{eq: epsilon}.
\end{proof}

\begin{lemma} \label{lem: convex polytope}
    Let $C \subset \Rd^d$ be a closed convex polytope enclosed by $k$ hyperplanes, which has finite positive Lebesgue measure in $\Rd^d$. 
    Then, for any $\varepsilon>0$, there exists a two-layer ReLU network $\Tc$ with the architecture $\ReLUtwo{d}{k}{1}$ such that
    \begin{align*}
        \Tc(\xb) &= 1 \qquad \text{if } \xb\in C, \\
        \Tc(\xb) &< 1 \qquad \text{if } \xb\in B_\varepsilon(C) - C. \\
        \Tc(\xb) &< 0 \qquad \text{if } \xb\not\in B_\varepsilon(C).
    \end{align*}
\end{lemma}
\begin{figure}[t]
    \centering
    \begin{subfigure}[b]{0.45\textwidth}
        \centering
        \includegraphics[width=\textwidth]{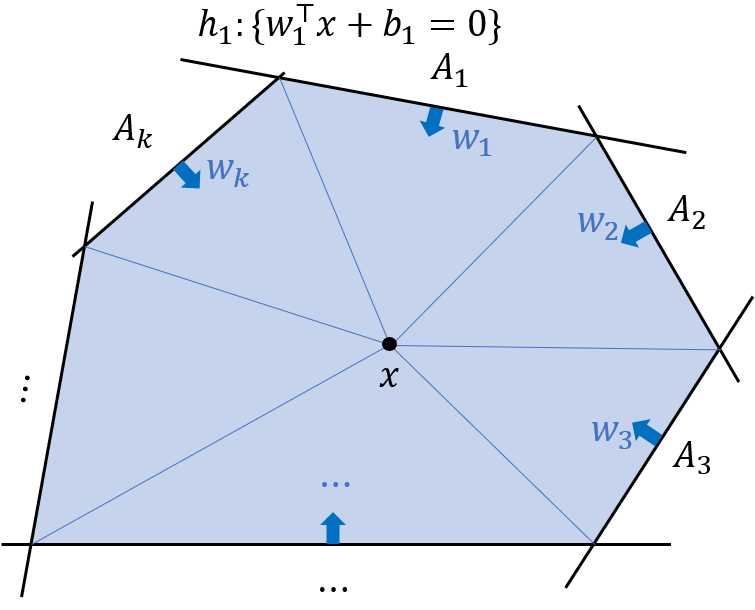}
        \caption{}
    \end{subfigure}
    \hfill
    \begin{subfigure}[b]{0.45\textwidth}
        \centering
        \includegraphics[width=\textwidth]{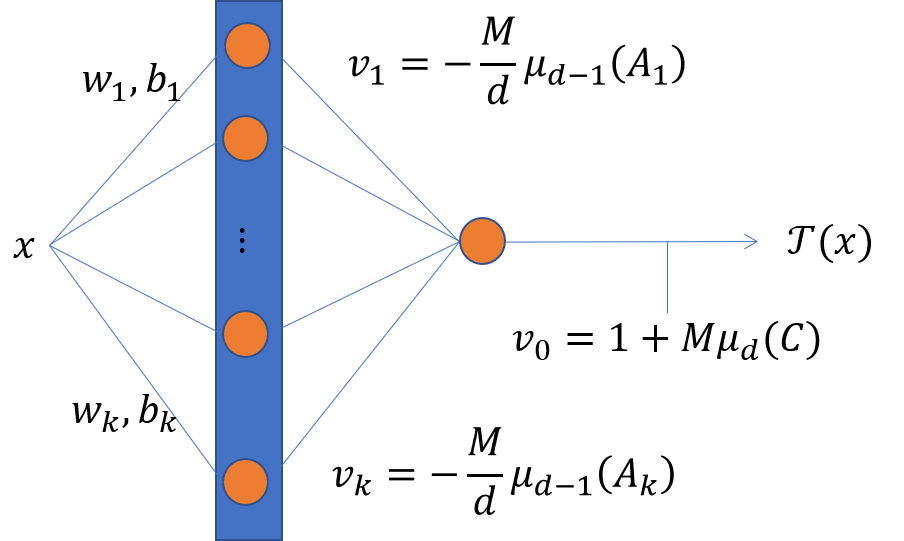}
        \caption{}
    \end{subfigure}
    \caption{Proof of Lemma \ref{lem: convex polytope}. 
    (a) A closed convex polytope $C\subset \Rd^d$ enclosed by $k$ hyperplanes. The volume of $C$ is the sum of volume of $k$ cones.
    (b) The architecture of the desired two-layer ReLU network $\Tc$: $\ReLUtwo{d}{k}{1}$.}
    \label{fig: convex polytope}
\end{figure}
\begin{proof}
    Let $h_1,\cdots,h_k$ be the $k$ hyperplanes enclosing $C$. Let $\wb_i$ be the unit normal vector of the $i$-th hyperplane $h_i$ oriented inside $C$. Then the equation of the $i$-th hyperplane $h_i$ is given by $h_i : \{\xb ~|~\wb_i^\top\xb+b_i=0\}$ for some $b_i\in\Rd$. Let $A_i$ be the intersection of the hyperplane $h_i$ and $C$, which is a face of the polytope $C$. 
    Let $\xb$ be a point in $C$. Since $\wb_i$ is a unit normal vector, $\wb_i^\top\xb+b_i$ refers the distance between the hyperplane $h_i$ and the point $\xb$. Therefore, the $d$-dimensional Lebesgue measure of $C$ is computed by
    \begin{align} \label{eq: volume}
        \mu_d(C) = \frac{1}{d} \sum_{i=1}^k  (\wb_i^\top\xb+b_i) \cdot\mu_{d-1}(A_i)
    \end{align}
    where $\mu_{d-1}$ and $\mu_{d}$ refer the $(d-1)$ and $d$-dimensional Lebesgue measures, respectively. Note that \eqref{eq: volume} comes from the volume formula of a cone.
    Then LHS of \eqref{eq: volume} is constant, which does not depend on the choice of $\xb\in\Rd^d$. Now, we define a two-layer ReLU network $\Tc$ with the architecture $\ReLUtwo{d}{k}{1}$ by
    \begin{align} \label{eq: polytope}
        \Tc(\xb) := 1 + M\left( \mu_d(C) -\sum_{i=1}^k \frac{1}{d} \mu_{d-1}(A_i) \cdot \sigma(\wb_i^\top\xb + b_i) \right)
    \end{align}
    where $M>0$ is a constant would be determined later. 
    Note that we have $\Tc(\xb) = 1$ for $\xb \in C$ from the construction. 
    It is worth noting that the equation \eqref{eq: volume} also holds for $\xb\not\in C$. For $\xb\not\in C$, \eqref{eq: polytope} deduces
    \begin{align*}
        \Tc(\xb) &= 1 + M\left(\mu_d(C) - \sum_{i=1}^k \frac{1}{d} \mu_{d-1}(A_i) \cdot \sigma(\wb_i^\top\xb+b_i)\right) \\
        &= 1 + M\left(\mu_d(C) - \sum_{i=1}^k \frac{1}{d} \mu_{d-1}(A_i) \cdot (\wb_i^\top\xb+b_i) 
        + \sum_{\{i~:~\wb_i^\top\xb + b_i < 0\}} \frac{1}{d} \mu_{d-1}(A_i) \cdot (\wb_i^\top\xb+b_i)\right) \\
        &= 1 + M \sum_{\{i~:~\wb_i^\top\xb + b_i < 0\}} \frac{1}{d} \mu_{d-1}(A_i) \cdot (\wb_i^\top\xb+b_i) \\
        &<1 .
    \end{align*}
    Therefore, we conclude
    \begin{align*}
        \Tc(\xb) &= 1 \qquad \text{if } \xb\in C, \\
        \Tc(\xb) &< 1 \qquad \text{otherwise}.
    \end{align*}
    Lastly, we determine the constant $M$ in $\Tc$ to satisfy the remained property. For the given $\varepsilon>0$, consider the closure of complement of the $\frac{\varepsilon}{2}$-neighborhood of $C$;  $D:=\overline{ \left(B_{{\varepsilon}/{2}}(C)\right)^c}$. Then the previsous result shows that
    \begin{align} \label{eq: u}
        \frac{1}{M}(\Tc(\xb) - 1) = \mu_d(C) - \sum_{i=1}^k \frac{1}{d} \mu_{d-1}(A_i) \cdot \sigma(\wb_i^\top\xb+b_i)
    \end{align}
    is bounded above by $0$. Furthermore, \eqref{eq: u} is continuous piecewise linear, and has the maximum $0$ if and only if $\xb\in C$. Since $D$ is closed and \eqref{eq: u} is strictly bounded above by $0$ on $D$, \eqref{eq: u} has the finite maximum $m<0$ on $D$.
    \begin{align*}
        \frac{1}{M}(\Tc(\xb) - 1) \le m < 0 \qquad \text{for } \xb \in D.
    \end{align*}
    
    Now, choose $M$ to satisfy $M>-\frac{1}{m}$. Then if $\xb \not\in B_{\varepsilon}(C)$, we have $\xb \in D$, thus
    \begin{align*}
        \Tc(\xb) &= 1 + M\left(\mu_d(C) - \sum_{i=1}^k \frac{1}{d} \mu_{d-1}(A_i) \cdot \sigma(\wb_i^\top\xb+b_i)\right) \\
        &\le 1+ M \cdot m \\
        &< 0.
    \end{align*}
    Therefore, we have constructed a two-layer ReLU network $\Tc$ with the structure $\ReLUtwo{d}{k}{1}$ such that
    \begin{align*}
        \Tc(\xb) &= 1 \qquad \text{if } \xb\in C, \\
        \Tc(\xb) &< 1 \qquad \text{if } \xb\in C^c. \\
        \Tc(\xb) &< 0 \qquad \text{if } \xb\not\in B_\varepsilon(C).
    \end{align*}
    which completes the proof.
\end{proof}

\begin{lemma} \label{lem: simplex}
    Let $0\le m \le d$ be integers, and $\Delta^m$ be an $m$-simplex in $\Rd^d$. For a given $\varepsilon>0$, there exists a two-layer ReLU network $\Tc: \Rd^d \rightarrow \Rd$ with the architecture $\ReLUtwo{d}{(d+1)}{1}$ such that
    \begin{align*}
        \Tc(\xb) &= 1       \qquad \text{if } \xb \in \Delta^m, \\
        \Tc(\xb) &\le 1     \qquad \text{if } \xb \in B_{\varepsilon}(\Delta^m), \\
        \Tc(\xb) &< 0       \qquad \text{if } \xb \not\in B_\varepsilon(\Delta^m).
    \end{align*}
    Furthermore, the minimal width of such two-layer ReLU networks with the architecture $\ReLUtwo{d}{d_1}{1}$ is exactly $d_1=d+1$.
\end{lemma}
\begin{proof}
    We prove the existence part first.
    For the given $m$-simplex $\Delta^m$, pick $(d-m)$ distinct points in $B_\varepsilon(\Delta^m)$. By connecting all these points with the points of $\Delta^m$, we obtain a $d$-simplex contained in $B_\varepsilon(\Delta^m)$, which is a convex polytope.
    By Lemma \ref{lem: convex polytope}, there exists a neural network $\Tc:\Rd^d \rightarrow \Rd$ with the architecture $\ReLUtwo{d}{d_1}{1}$ that satisfies the desired properties.
    
    Now, we prove the minimality part. For every $\varepsilon>0$, suppose there exists a two-layer ReLU network $\Tc(\xb) := \sum_{i=1}^{d_1} v_i \sigma( \wb_i^\top \xb+b_i) + v_0$ with $d_1 \le d$ such that $\Tc(\xb)=1$ for $\xb\in\Delta^m$ and  $\Tc(\xb)<0$ for $\xb \not\in B_\varepsilon(\Delta^m)$.
    First, we claim that the set of weight vectors $\{\wb_1, \cdots, \wb_{d_1}\}$ spans $\Rd^d$. If the set cannot span $\Rd^d$, then there exists a nonzero vector $\ub \in \Rd^d - \text{span}<\wb_1, \cdots, \wb_{d_1}>$. Then, from $\Tc(\xb)=1$ for $\xb\in \Delta^m$, we get
    \begin{align*}
        \Tc(\xb+t\ub) &= \sum_{i=1}^{d_1} v_i \sigma(\wb_i^\top(\xb+t\ub)+b_i) + v_0 \\
        &= \sum_{i=1}^{d_1} v_i \sigma(\wb_i^\top\xb+b_i) + v_0 \\
        &= \Tc(\xb) \\
        &=1
    \end{align*}
    for any $t\in\Rd$. This contradicts to the condition $\Tc(\xb)<0$ for $\xb \not \in B_\varepsilon(\Delta^m)$. Therefore, the set of weight vectors must span $\Rd^d$.

    From the above claim, we further deduce that $d_1 \ge d$. Since we start with the assumption $d_1 \le d$, thus $d_1=d$. Then, we conclude that the set of weight vectors $\{\wb_1, \cdots, \wb_{d_1}\}$ is a basis of $\Rd^d$.
    Now, we focus on the sign of $v_0$. Suppose $v_0 \ge 0$. Define 
    \begin{align*}
        A:= \bigcap_{i=1}^{d_1} \{\xb ~|~ \wb_i^\top\xb+b_i <0 \},
    \end{align*}
    which is an unbounded set since the set $\{\wb_i\}$ is linearly independent. Then for $\xb\in A$, we get $\Tc(\xb)=v_0 \ge 0$. This contradicts to the assumption $\Tc(\xb)<0$ for all $\xb \not \in B_\varepsilon(\Delta^m)$. 
    Therefore, $v_0<0$. 
    
    Lastly, we consider the sign of $v_i$. Since $\Tc(\xb)=1>0$ for $\xb\in \Delta^m$ and $v_0<0$, there exists some positive $v_i>0$, say, $v_1>0$.
    Similar to the above argument, we define
    \begin{align*}
        B:= \left\{\xb ~|~ v_1\wb_1^\top\xb+b_1 +v_0 >0 \right\}\; \bigcap_{i=2}^{d_1} \left\{\xb ~|~ \wb_i^\top\xb+b_i <0 \right\},
    \end{align*}
    which is also nonempty and unbounded. Then, for $\xb\in B$, we have 
    \begin{align*} 
        \Tc(\xb) &= \sum_{i=1}^{d_1} v_i \sigma(\wb_i^\top\xb+b_i)+v_0 \\
        &= v_1 \wb_1^\top\xb+b_1 + v_0 \\
        &>0.
    \end{align*}
    Since $B$ is unbounded, this implies that $\Tc(\xb)>0$ over the unbounded subset in $\Rd^d$, which contradicts to the condition $\Tc(\xb)<0$ for all $\xb\not\in B_\varepsilon(\Delta^m)$. This completes the whole proof, which shows that the minimum width of two-layer ReLU network is exactly $d+1$.
\end{proof}

\begin{figure}[t]
    \centering
    \begin{subfigure}[b]{0.36\textwidth}
        \centering
        \includegraphics[width=\textwidth]{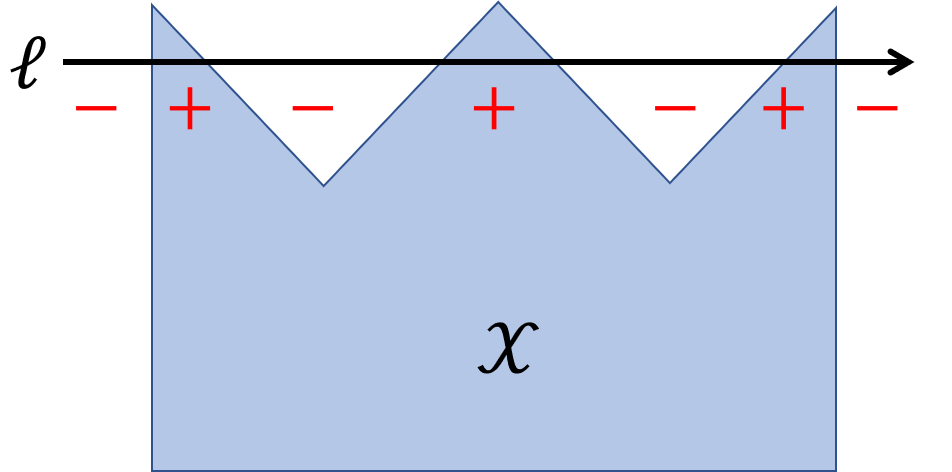}
        \caption{}
    \end{subfigure}
    \hfill
    \begin{subfigure}[b]{0.36\textwidth}
        \centering
        \includegraphics[width=\textwidth]{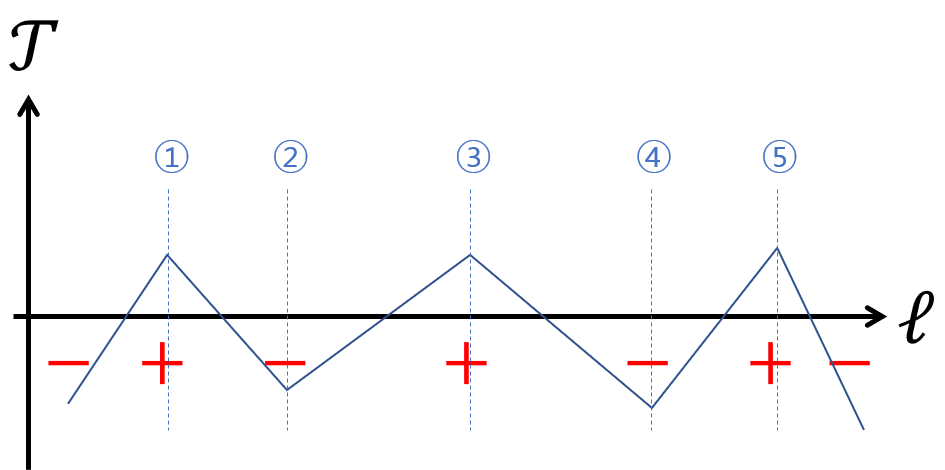}
        \caption{}
    \end{subfigure}
    \hfill
    \begin{subfigure}[b]{0.2\textwidth}
        \centering
        \includegraphics[width=\textwidth]{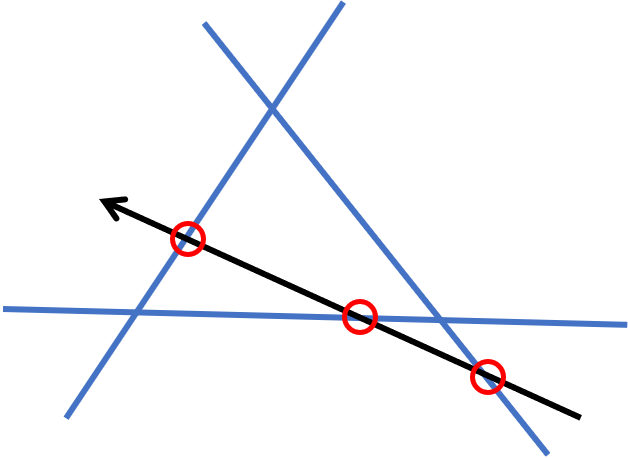}
        \caption{}
    \end{subfigure}
    \caption{Proof of Proposition \ref{prop: non convex}.
    (a) $\Xc$ is a crown-shaped topological space in $\Rd^2$. A straight line $\ell$ through the top of the crown. 
    (b) The pre-activation $\Tc(\xb)$ along the line $\ell$ must experience six sign changes, which needs at least five non-differentiable points.
    (c) However, any straight line in $\Rd^2$ can meet at most three non-differentiable points of $\Tc(\xb)$.}
    \label{fig: non convex}
\end{figure}
\begin{proposition} \label{prop: non convex}
    Let $\Xc$ be a crown-shaped topological space in $\Rd^2$, as shown in Figure \ref{fig: non convex}(a). Then, for some $\varepsilon>0$, there is no two-layer ReLU network $\Nc$ with the architecture $\ReLUTwo{2}{3}{1}$ such that
    \begin{align} \label{eq: universal}
        \norm{\Nc(\xb) - \indicator{\Xc}(\xb)}_{L^p(\Rd^2)} < \varepsilon.
    \end{align}
\end{proposition}
\begin{proof}
    We use proof by contradiction. For every $\varepsilon>0$, suppose there exists a two-layer ReLU network $\Nc(\xb) = \sigma(\sum_{i=1}^{3}v_i\sigma(\wb_i^\top\xb+b_i) +v_0)$ that satisfies \eqref{eq: universal}. 
    Let $\Tc(\xb) := \sum_{i=1}^{k}v_i\sigma(\wb_i^\top\xb+b_i) +v_0$ be the pre-activation of the $\Nc(\xb)$. 
    
    Now, consider a straight line $\ell$ that through the sawtooth part of $\Xc$, as described in \ref{fig: non convex}(a). Since $\Tc(\xb)$ is piecewise linear, it is still piecewise linear on the line $\ell$. However, since $\Nc(\xb)$ is a sufficiently good approximator of $\indicator{\Xc}$, we have $|\Nc(\xb) - \indicator{\Xc}|$ is small, which implies $\Tc(\xb) \le 0$ for $\xb \not\in \B_\varepsilon(\Xc)$ and $\Tc(\xb)>0$ for $\xb \in \Nc(\xb)$. Therefore, $\Tc(\xb)$ on the line $\ell$ must have alternative negative and positive values as shown in \ref{fig: non convex}(a). More precisely, $\Tc(\xb)|_{\xb\in\ell}$ has four negative parts and three positive parts. Since it is piecewise linear, to change the sign six times, it requires at least five flipping points as shown in Figure \ref{fig: non convex}(b). In other words, $\Tc(\xb)$ has at least five non-differentiable points.
    However, $\Tc(\xb)$ can have at most three non-differentiable points on any straight line because it has only three ReLU neurons (Figure \ref{fig: non convex}(c)). This contradiction proves that there is no two-layer ReLU network $\Nc$ with the architecture $\ReLUTwo{2}{3}{1}$ that can approximate $\indicator{\Xc}$ arbitrarily close.
\end{proof}

\vspace{0.5cm}
\begin{remark} \label{rmk: Betti numbers}
    Proposition \ref{prop: compact} and Proposition \ref{prop: non convex} demonstrate that the minimum widths of neural networks can vary for two topologically equivalent spaces. In $\Rd^2$, Proposition \ref{prop: compact} illustrates that the indicator function on a solid triangle can be approximated by a network with the architecture $\ReLUTwo{2}{3}{1}$. However, Proposition \ref{prop: non convex} shows that the indicator function on a crown-shaped topological space cannot be approximated by a neural network with the architecture $\ReLUTwo{2}{3}{1}$. Despite the two topological spaces being homotopy equivalent and having the same Betti numbers, the minimum architecture of the neural networks are different.
\end{remark}

\section{Further Experimental Results} \label{app: experiment}

In this section, we present additional experimental results that specifically investigate the effect of initialization.
We use the two datasets $\Xc_1$ and $\Xc_2$ described in Figure \ref{fig: experiments}(a) and (d). 
For three or four-layer ReLU networks employing the MSE loss during training, which exhibit ReLU or $\texttt{MAX}$ activation at the final layer, the networks often encounter dying ReLU initialization issue (constant output) as mentioned in \cite{lu2019dying}. Therefore, we have excluded these cases from our analysis.

In Figure \ref{fig: initialization TT} and \ref{fig: initialization HEX}, we present the experimental results for various initialization schemes.
The first and third columns represent the initialization state of the network, where the second and fourth columns show the trained network. Each row has the same initialization schemes, which are uniform initialization $U([0,1])$, normal initialization $N(0,\Ib)$, Xavier uniform and normal initialization \cite{glorot2010understanding}, He uniform and normal initialization \cite{he2015delving}, small-norm initialization, and manually initialized weight and bias, from top to bottom. 
Note that the last row represents manually initialized case, which are intended to converge to the network we constructed in the proof of Proposition \ref{prop: compact} and Theorem \ref{thm: compact}. Indeed, weight vectors in the first layer of the converged networks enclose each manifold $\Xc_i$. This experimental results show that our proposed networks can be achieved by gradient descent under specific initialization. 


\subsection{Detail setup of experiments}
We outline the detailed setup of the experiments here. All experiments were executed using Pytorch on a GeForce GTX 1080 Ti. For each task, the training dataset was generated from $40\times40=1600$ lattice points in the $[-20,20]\times[-20,20] \subset \Rd^2$ range, where each point was labeled `1' if the point was in $\Xc$, and `0' otherwise.

We utilized two types of loss functions. For the training set ${(\xb_i,y_i)}_{i=1}^n$, the mean square error (MSE) loss is defined as:
\begin{align*}
    L_{MSE} (\theta) := \frac{1}{n} \sum_{i=1}^n \left[ \Nc_\theta(\xb_i) - y_i)\right]^2.
\end{align*}
The binary cross entropy (BCE) loss is given by:
\begin{align*}
    L_{BCE} (\theta) := \frac{1}{n} \sum_{i=1}^n \left[ y_i\log(\Nc_\theta(\xb_i)) + (1-y_i)\log(1-\Nc_\theta(\xb_i)) \right].
\end{align*}
It should be noted that the input of the BCE loss must satisfy $0<\Nc_\theta(\xb_i)<1$, which necessitates the use of the sigmoid activation function in the last layer.

The optimization method employed was full-batch gradient descent, with learning rates set at 0.005 for $\Xc_1$ (two triangles) and 0.001 for $\Xc_2$ (a hexagon with a pentagon hole). The total number of epochs was flexibly determined for each experiment, ensuring a sufficient number to achieve convergence.

\begin{figure}
    \centering
    \begin{subfigure}[b]{0.23\textwidth}
        \centering
        \includegraphics[width=\textwidth]{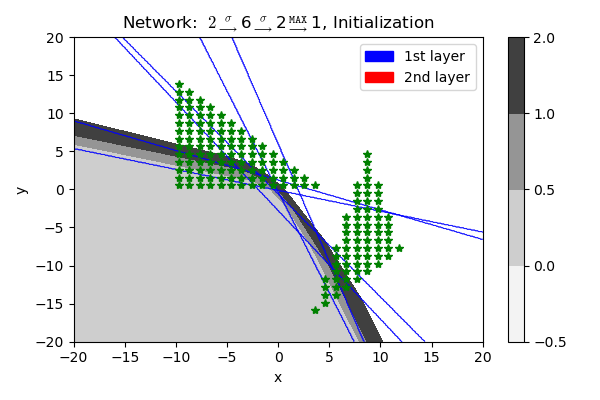}
    \end{subfigure}
    \hfill 
    \begin{subfigure}[b]{0.23\textwidth}
        \centering
        \includegraphics[width=\textwidth]{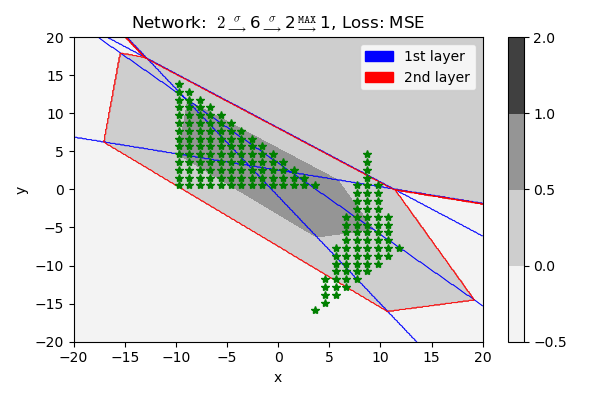}
    \end{subfigure}
    \hfill 
    \begin{subfigure}[b]{0.23\textwidth}
        \centering
        \includegraphics[width=\textwidth]{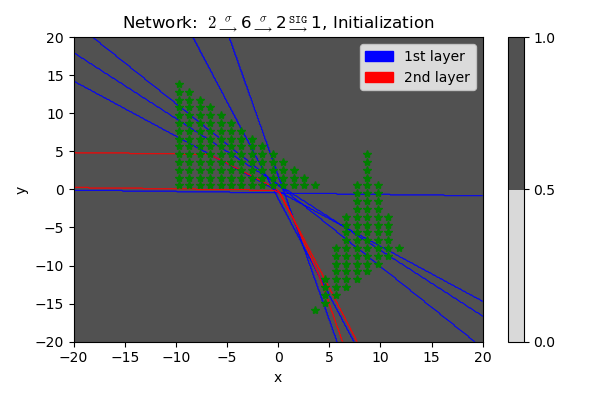}
    \end{subfigure}
    \hfill 
    \begin{subfigure}[b]{0.23\textwidth}
        \centering
        \includegraphics[width=\textwidth]{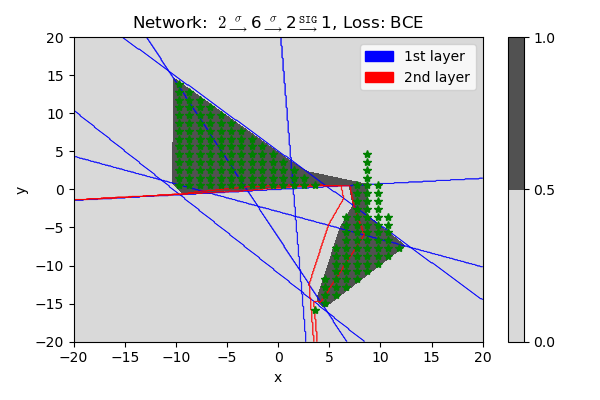}
    \end{subfigure}
    \\
    \begin{subfigure}[b]{0.23\textwidth}
        \centering
        \includegraphics[width=\textwidth]{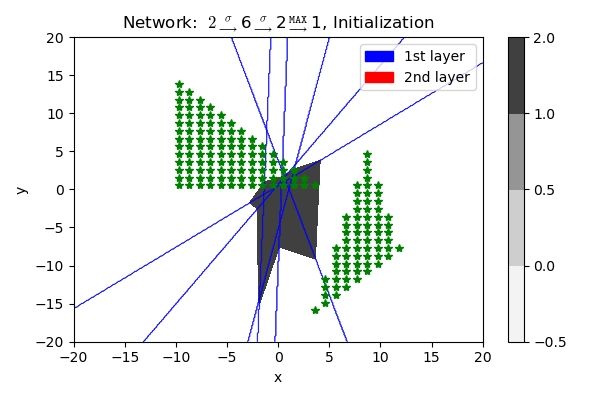}
    \end{subfigure}
    \hfill 
    \begin{subfigure}[b]{0.23\textwidth}
        \centering
        \includegraphics[width=\textwidth]{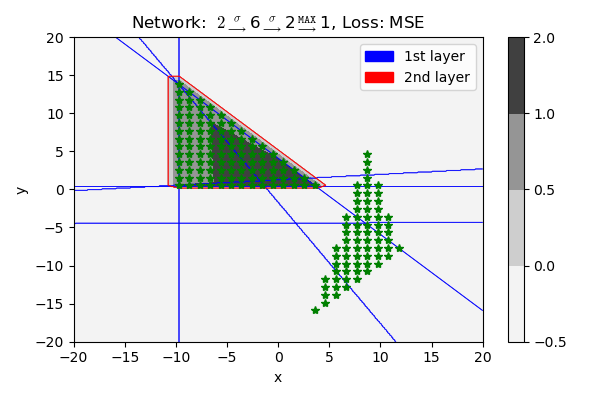}
    \end{subfigure}
    \hfill 
    \begin{subfigure}[b]{0.23\textwidth}
        \centering
        \includegraphics[width=\textwidth]{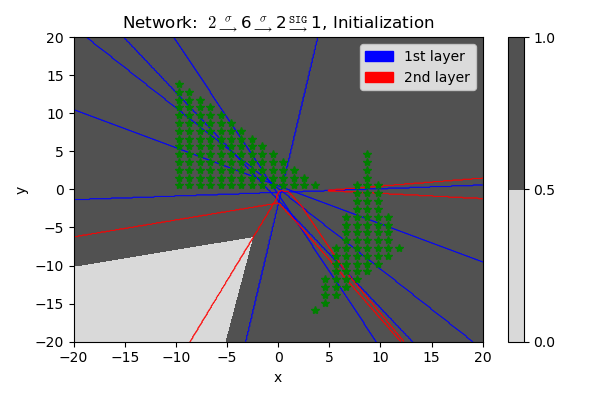}
    \end{subfigure}
    \hfill 
    \begin{subfigure}[b]{0.23\textwidth}
        \centering
        \includegraphics[width=\textwidth]{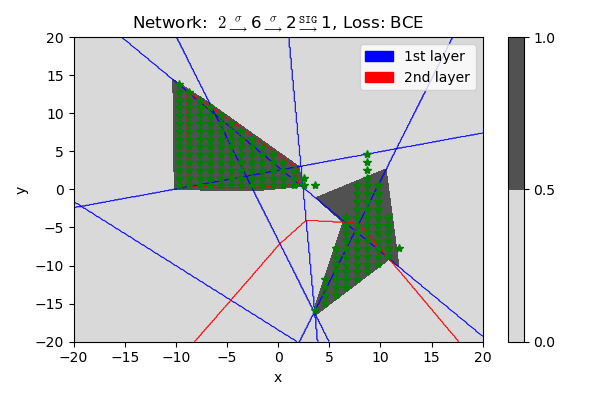}
    \end{subfigure}
    \\
    \begin{subfigure}[b]{0.23\textwidth}
        \centering
        \includegraphics[width=\textwidth]{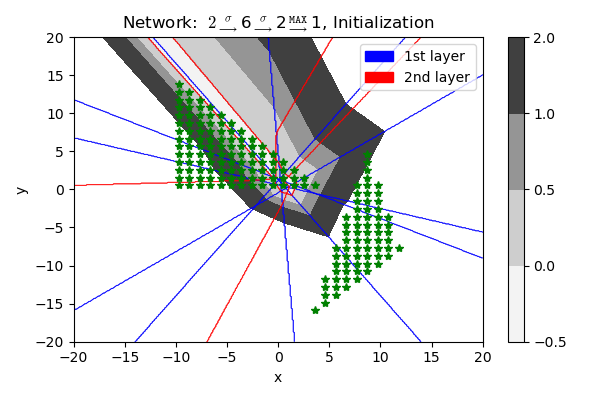}
    \end{subfigure}
    \hfill 
    \begin{subfigure}[b]{0.23\textwidth}
        \centering
        \includegraphics[width=\textwidth]{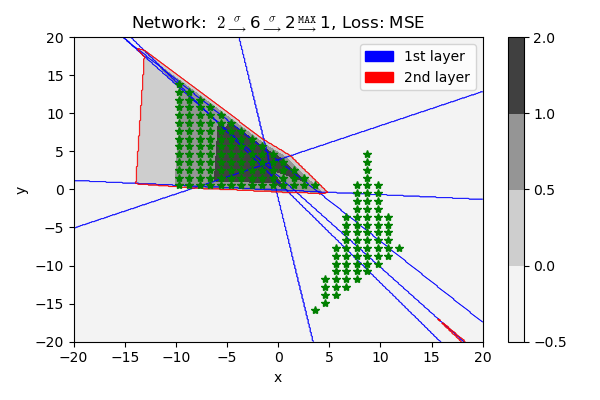}
    \end{subfigure}
    \hfill 
    \begin{subfigure}[b]{0.23\textwidth}
        \centering
        \includegraphics[width=\textwidth]{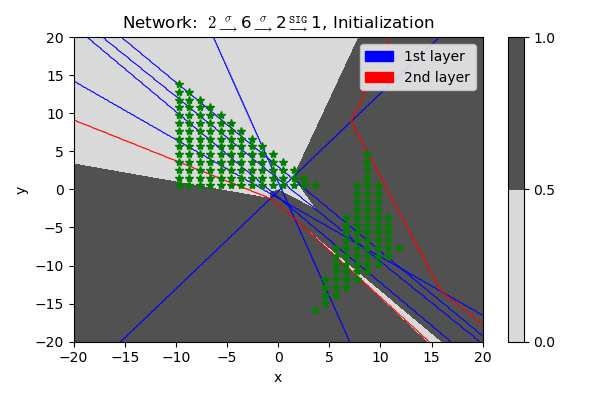}
    \end{subfigure}
    \hfill 
    \begin{subfigure}[b]{0.23\textwidth}
        \centering
        \includegraphics[width=\textwidth]{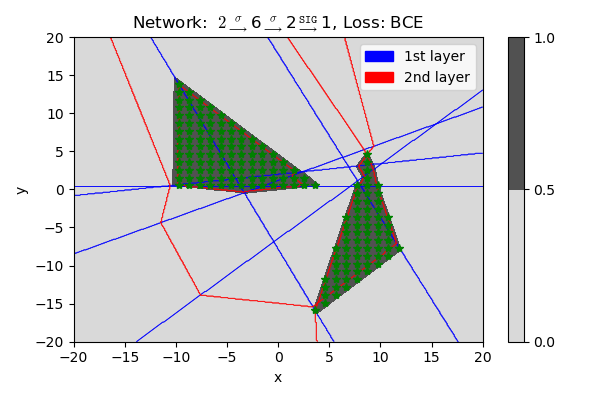}
    \end{subfigure}
    \\
    \begin{subfigure}[b]{0.23\textwidth}
        \centering
        \includegraphics[width=\textwidth]{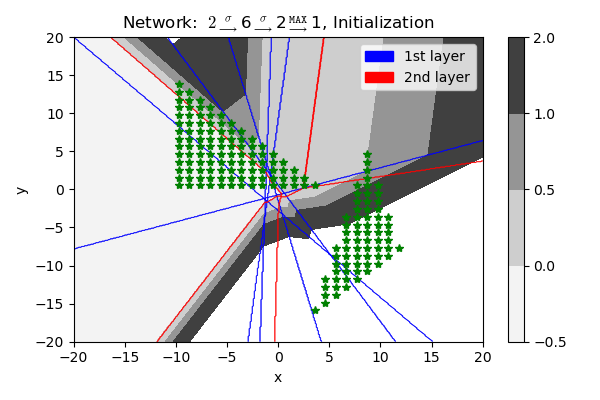}
    \end{subfigure}
    \hfill 
    \begin{subfigure}[b]{0.23\textwidth}
        \centering
        \includegraphics[width=\textwidth]{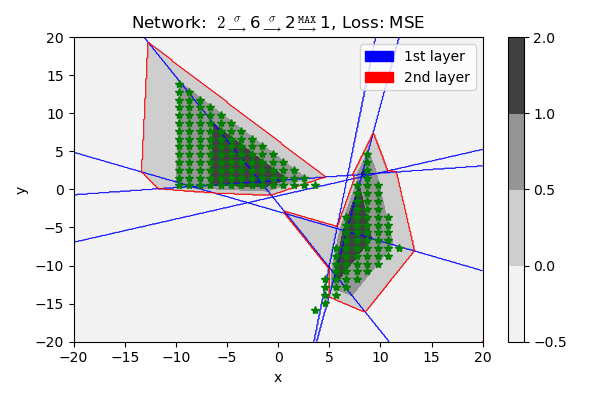}
    \end{subfigure}
    \hfill 
    \begin{subfigure}[b]{0.23\textwidth}
        \centering
        \includegraphics[width=\textwidth]{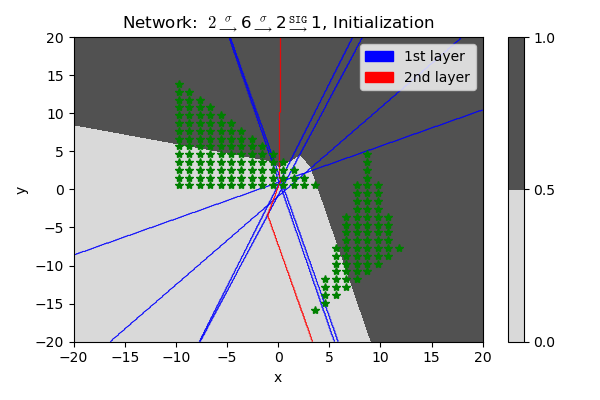}
    \end{subfigure}
    \hfill 
    \begin{subfigure}[b]{0.23\textwidth}
        \centering
        \includegraphics[width=\textwidth]{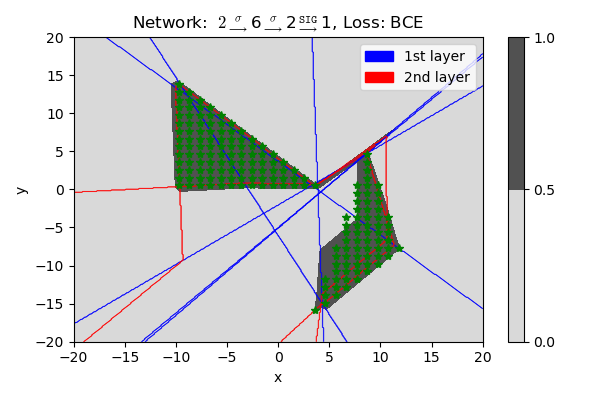}
    \end{subfigure}
    \\
    \begin{subfigure}[b]{0.23\textwidth}
        \centering
        \includegraphics[width=\textwidth]{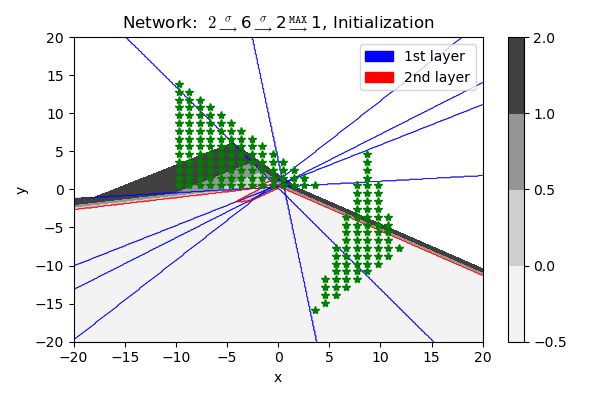}
    \end{subfigure}
    \hfill 
    \begin{subfigure}[b]{0.23\textwidth}
        \centering
        \includegraphics[width=\textwidth]{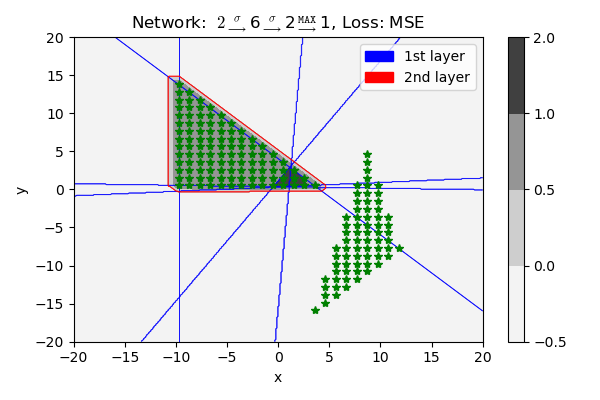}
    \end{subfigure}
    \hfill 
    \begin{subfigure}[b]{0.23\textwidth}
        \centering
        \includegraphics[width=\textwidth]{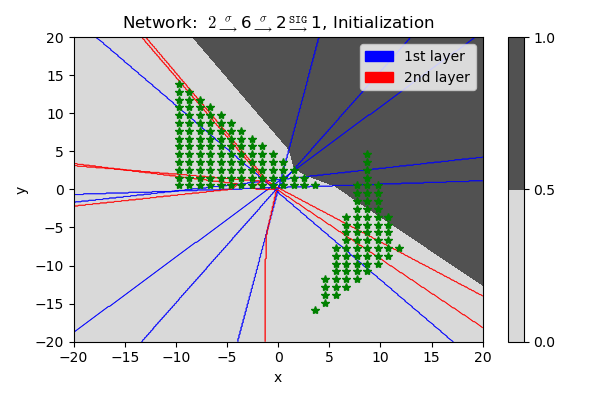}
    \end{subfigure}
    \hfill 
    \begin{subfigure}[b]{0.23\textwidth}
        \centering
        \includegraphics[width=\textwidth]{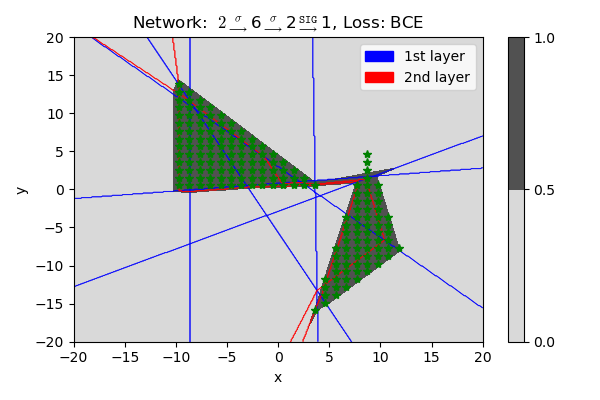}
    \end{subfigure}
    \\
    \begin{subfigure}[b]{0.23\textwidth}
        \centering
        \includegraphics[width=\textwidth]{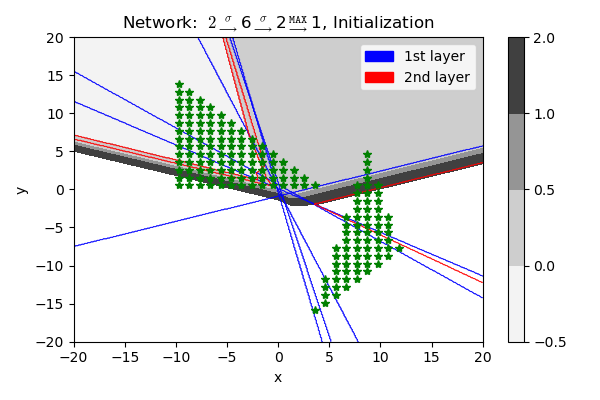}
    \end{subfigure}
    \hfill 
    \begin{subfigure}[b]{0.23\textwidth}
        \centering
        \includegraphics[width=\textwidth]{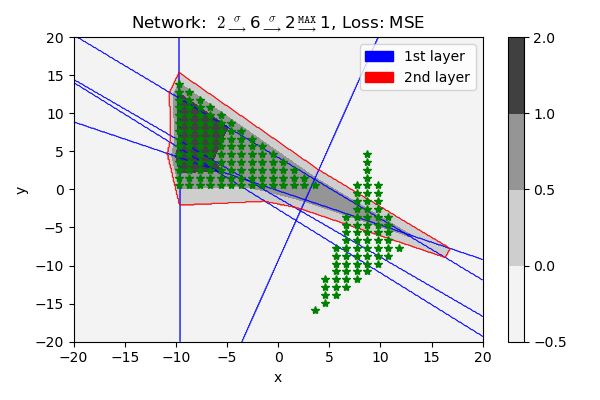}
    \end{subfigure}
    \hfill 
    \begin{subfigure}[b]{0.23\textwidth}
        \centering
        \includegraphics[width=\textwidth]{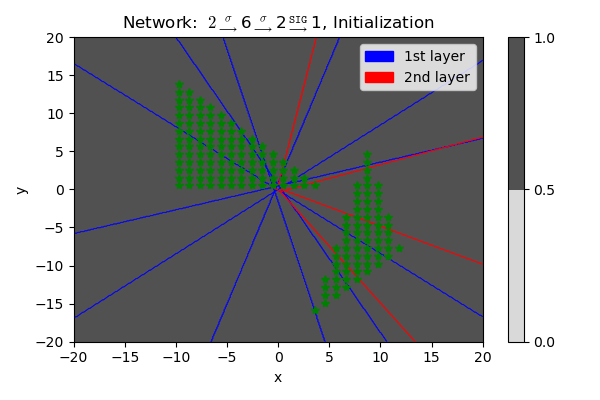}
    \end{subfigure}
    \hfill 
    \begin{subfigure}[b]{0.23\textwidth}
        \centering
        \includegraphics[width=\textwidth]{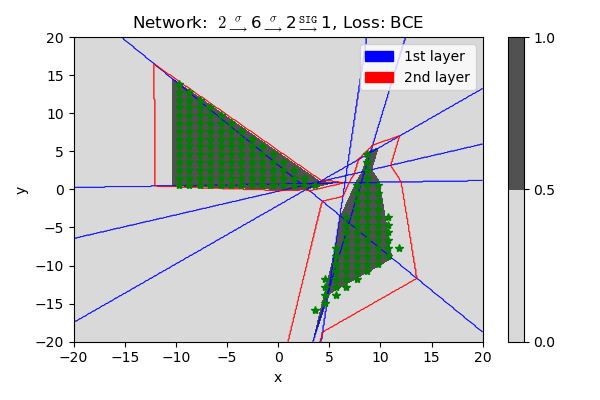}
    \end{subfigure}
    \\
    \begin{subfigure}[b]{0.23\textwidth}
        \centering
        \includegraphics[width=\textwidth]{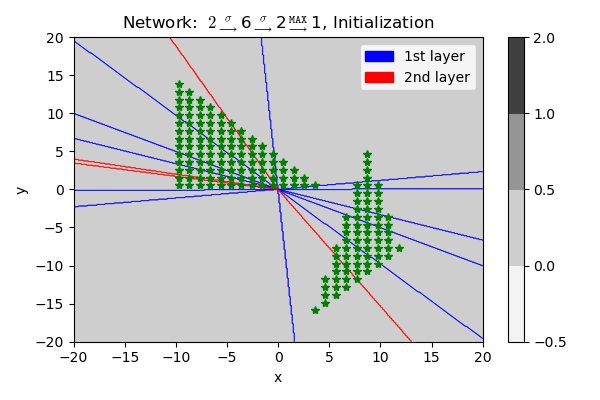}
    \end{subfigure}
    \hfill 
    \begin{subfigure}[b]{0.23\textwidth}
        \centering
        \includegraphics[width=\textwidth]{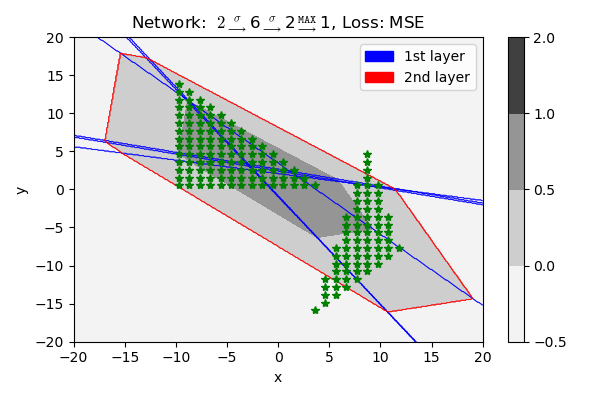}
    \end{subfigure}
    \hfill 
    \begin{subfigure}[b]{0.23\textwidth}
        \centering
        \includegraphics[width=\textwidth]{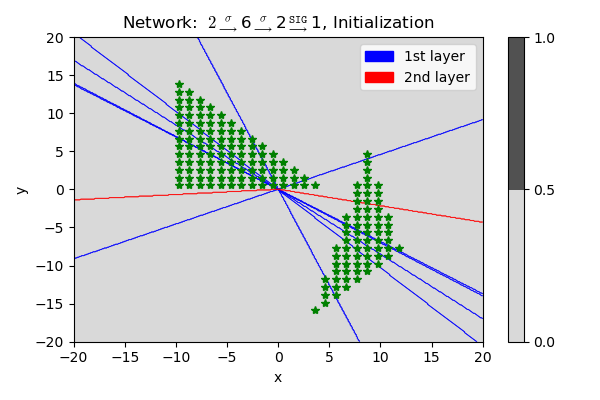}
    \end{subfigure}
    \hfill 
    \begin{subfigure}[b]{0.23\textwidth}
        \centering
        \includegraphics[width=\textwidth]{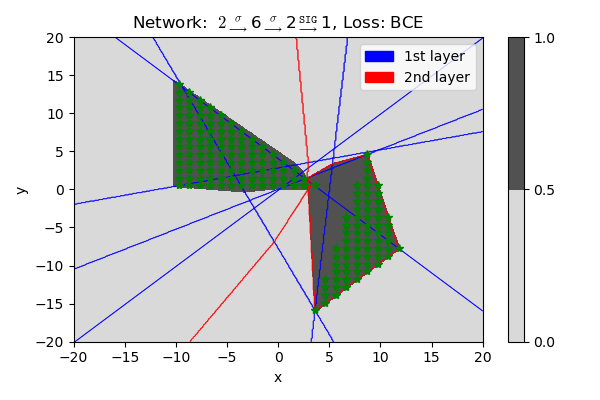}
    \end{subfigure}
    \\
    \begin{subfigure}[b]{0.23\textwidth}
        \centering
        \includegraphics[width=\textwidth]{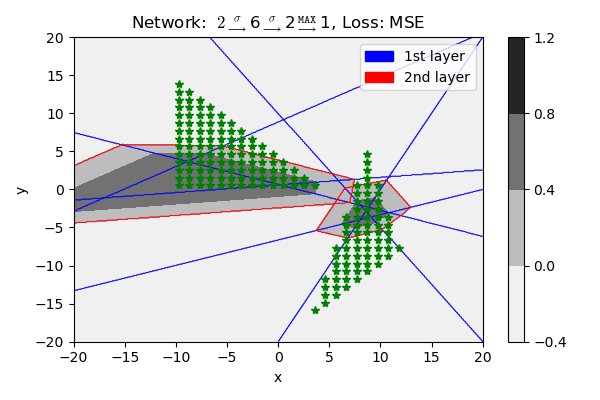}
    \end{subfigure}
    \hfill 
    \begin{subfigure}[b]{0.23\textwidth}
        \centering
        \includegraphics[width=\textwidth]{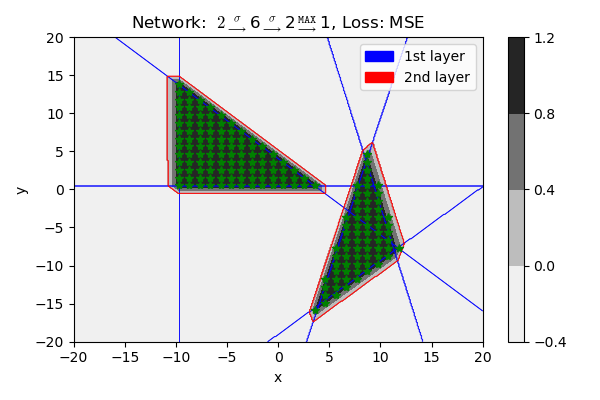}
    \end{subfigure}
    \hfill 
    \begin{subfigure}[b]{0.23\textwidth}
        \centering
        \includegraphics[width=\textwidth]{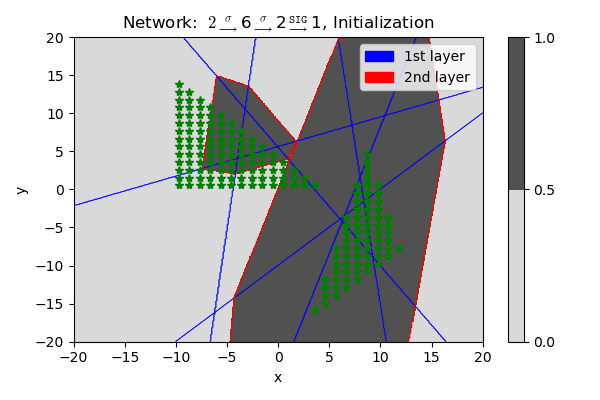}
    \end{subfigure}
    \hfill 
    \begin{subfigure}[b]{0.23\textwidth}
        \centering
        \includegraphics[width=\textwidth]{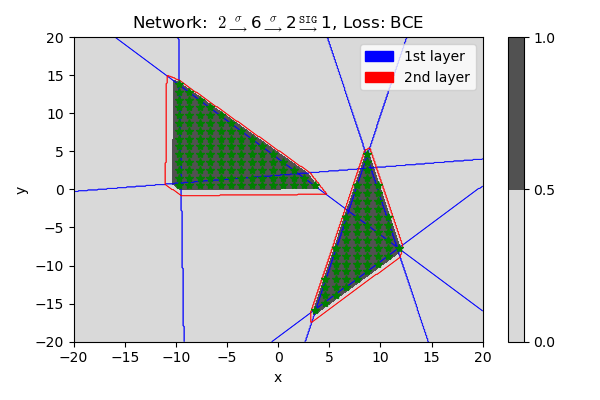}
    \end{subfigure}
    \caption{Initialization and convergence results to approximate $\indicator{\Xc_2}$ via gradient descent. 
    The first and third columns present the network initialization, and the second and fourth columns present the trained network under MSE and BCE loss, respectively.
    Each row has the same initialization scheme, which is given by uniform initialization $U([0,1])$, normal initialization $N(0,\Ib)$, Xavier uniform and normal initialization \cite{glorot2010understanding}, He uniform and normal initialization \cite{he2015delving}, small-norm initialization, and manually initialized weights and bias, from top to bottom.}
    \label{fig: initialization TT}
\end{figure}
\begin{figure}
    \centering
    \begin{subfigure}[b]{0.23\textwidth}
        \centering
        \includegraphics[width=\textwidth]{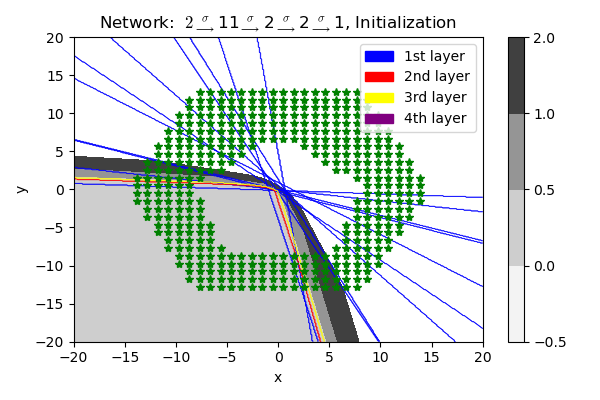}
    \end{subfigure}
    \hfill 
    \begin{subfigure}[b]{0.23\textwidth}
        \centering
        \includegraphics[width=\textwidth]{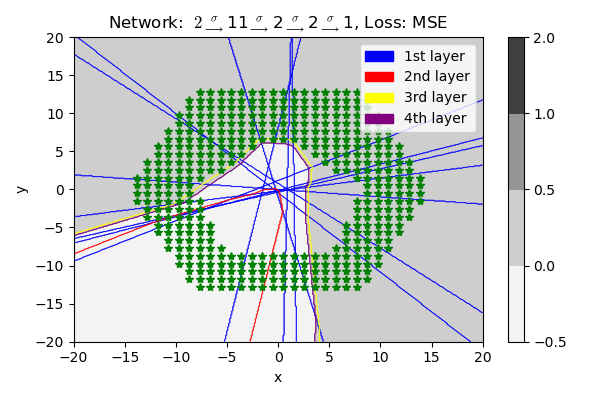}
    \end{subfigure}
    \hfill 
    \begin{subfigure}[b]{0.23\textwidth}
        \centering
        \includegraphics[width=\textwidth]{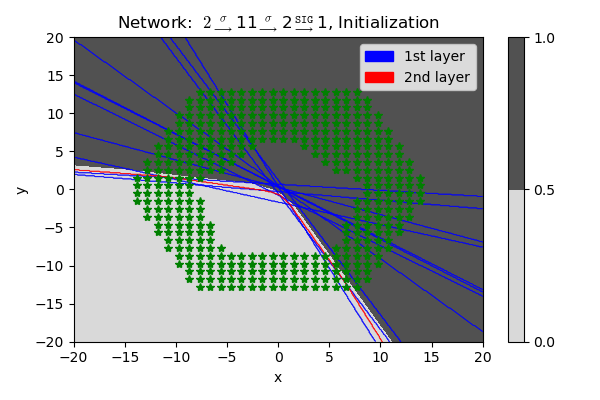}
    \end{subfigure}
    \hfill 
    \begin{subfigure}[b]{0.23\textwidth}
        \centering
        \includegraphics[width=\textwidth]{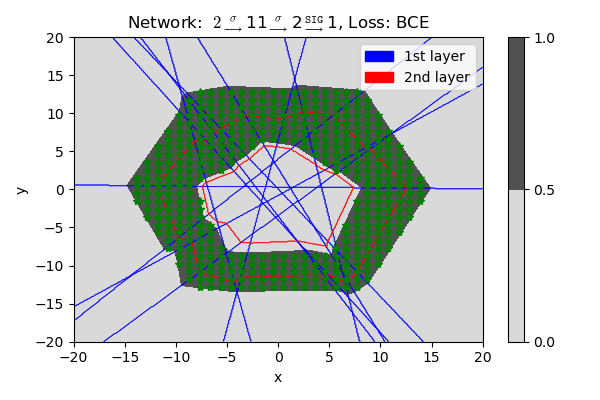}
    \end{subfigure}
    \\
    \begin{subfigure}[b]{0.23\textwidth}
        \centering
        \includegraphics[width=\textwidth]{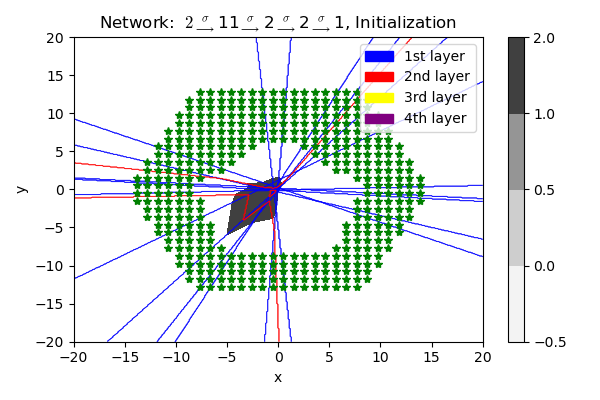}
    \end{subfigure}
    \hfill 
    \begin{subfigure}[b]{0.23\textwidth}
        \centering
        \includegraphics[width=\textwidth]{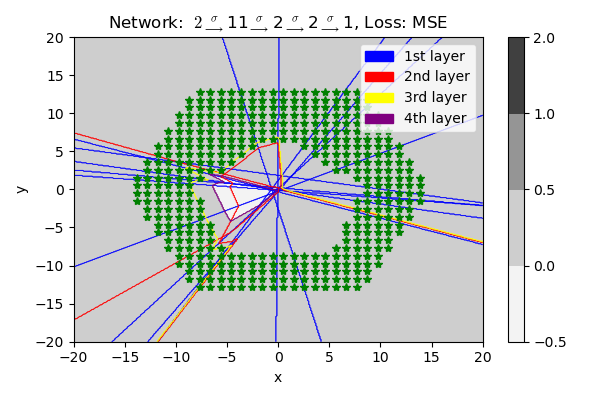}
    \end{subfigure}
    \hfill 
    \begin{subfigure}[b]{0.23\textwidth}
        \centering
        \includegraphics[width=\textwidth]{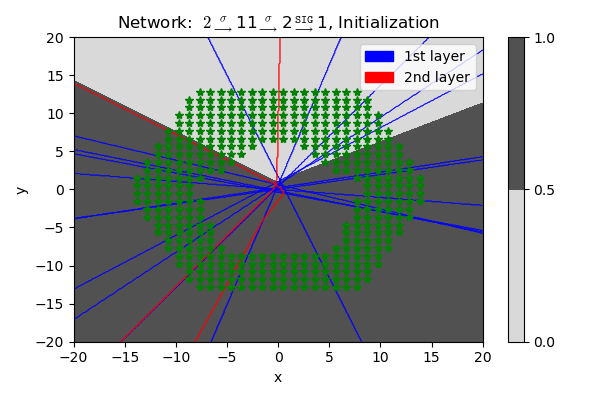}
    \end{subfigure}
    \hfill 
    \begin{subfigure}[b]{0.23\textwidth}
        \centering
        \includegraphics[width=\textwidth]{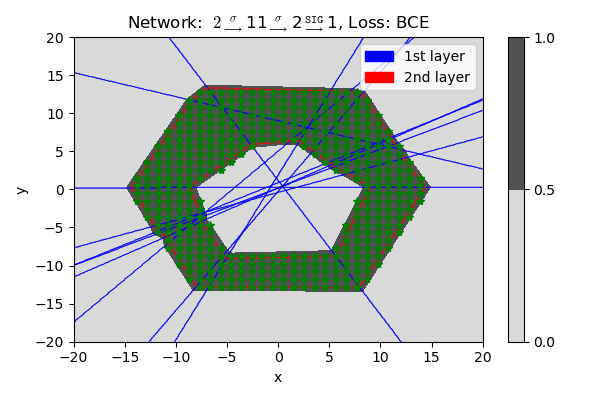}
    \end{subfigure}
    \\
    \begin{subfigure}[b]{0.23\textwidth}
        \centering
        \includegraphics[width=\textwidth]{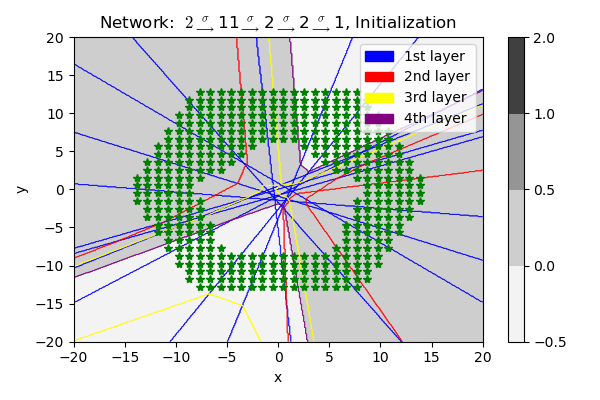}
    \end{subfigure}
    \hfill 
    \begin{subfigure}[b]{0.23\textwidth}
        \centering
        \includegraphics[width=\textwidth]{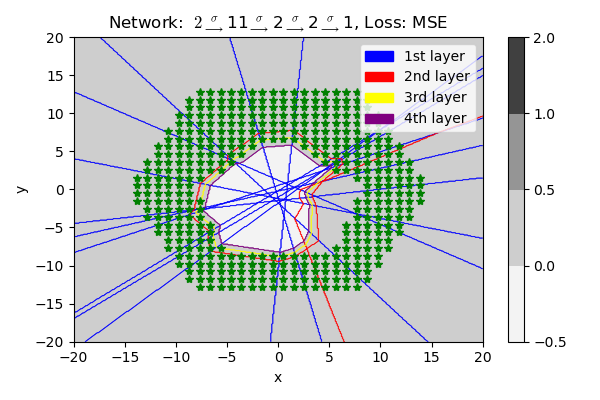}
    \end{subfigure}
    \hfill 
    \begin{subfigure}[b]{0.23\textwidth}
        \centering
        \includegraphics[width=\textwidth]{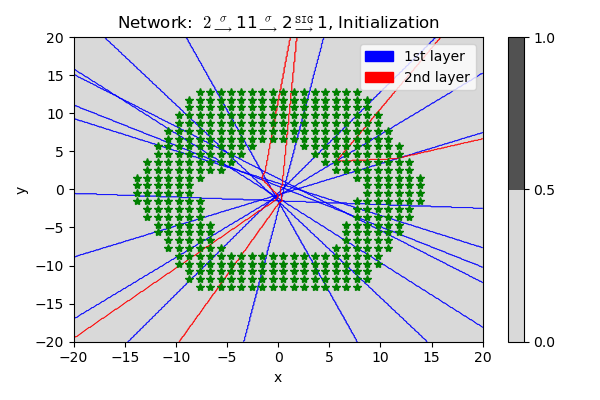}
    \end{subfigure}
    \hfill 
    \begin{subfigure}[b]{0.23\textwidth}
        \centering
        \includegraphics[width=\textwidth]{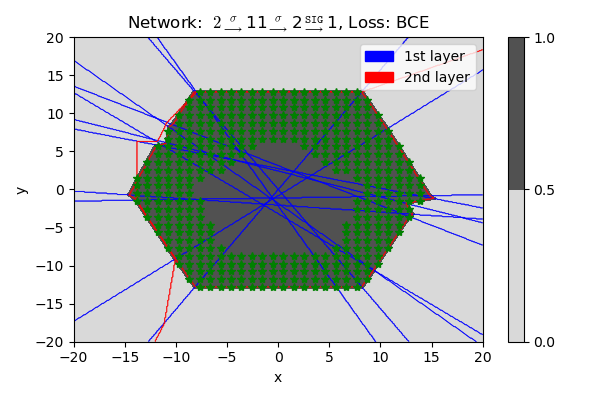}
    \end{subfigure}
    \\
    \begin{subfigure}[b]{0.23\textwidth}
        \centering
        \includegraphics[width=\textwidth]{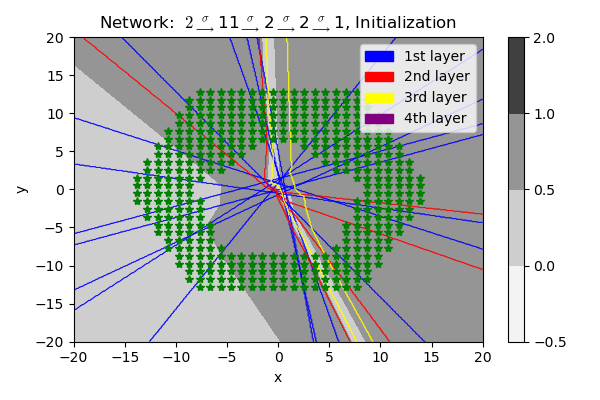}
    \end{subfigure}
    \hfill 
    \begin{subfigure}[b]{0.23\textwidth}
        \centering
        \includegraphics[width=\textwidth]{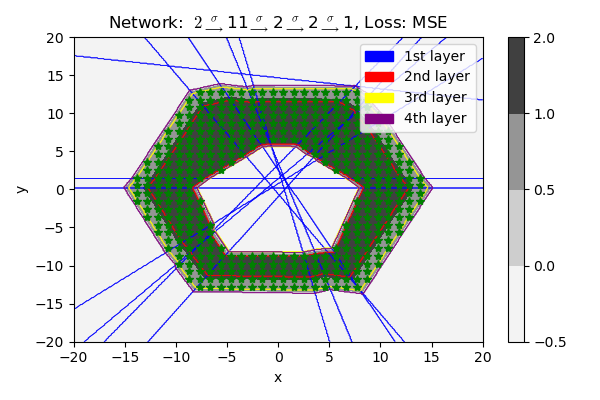}
    \end{subfigure}
    \hfill 
    \begin{subfigure}[b]{0.23\textwidth}
        \centering
        \includegraphics[width=\textwidth]{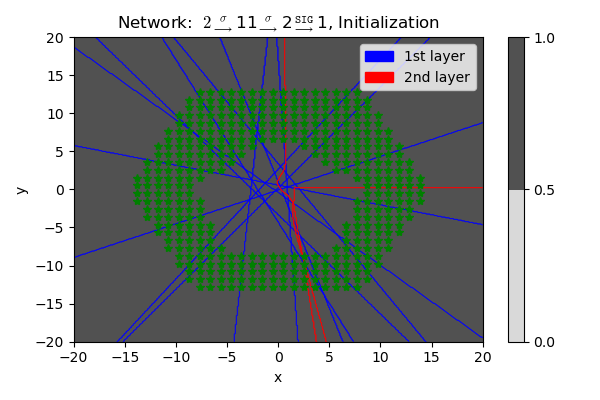}
    \end{subfigure}
    \hfill 
    \begin{subfigure}[b]{0.23\textwidth}
        \centering
        \includegraphics[width=\textwidth]{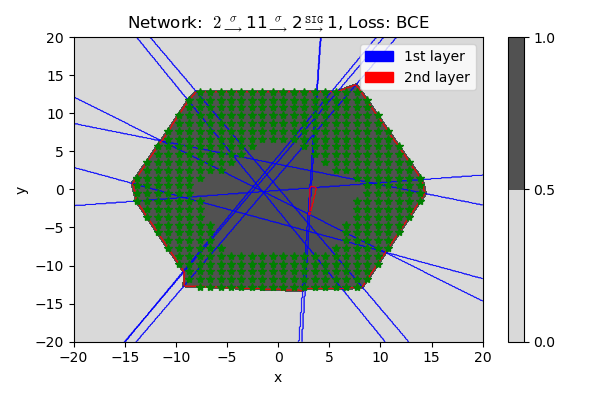}
    \end{subfigure}
    \\
    \begin{subfigure}[b]{0.23\textwidth}
        \centering
        \includegraphics[width=\textwidth]{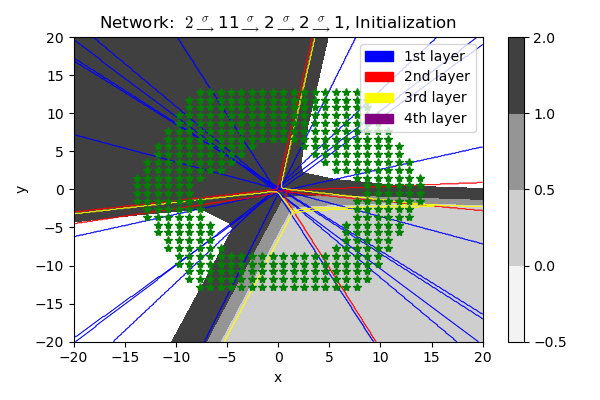}
    \end{subfigure}
    \hfill 
    \begin{subfigure}[b]{0.23\textwidth}
        \centering
        \includegraphics[width=\textwidth]{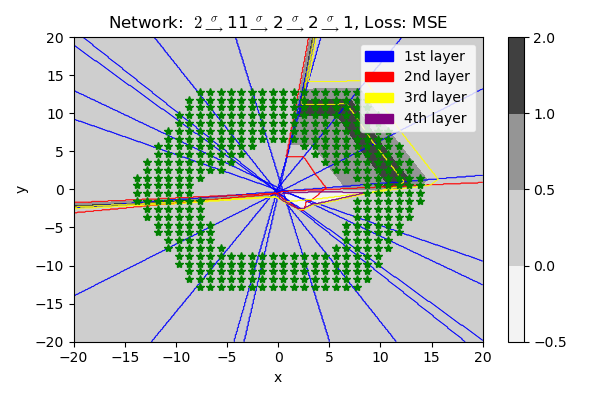}
    \end{subfigure}
    \hfill 
    \begin{subfigure}[b]{0.23\textwidth}
        \centering
        \includegraphics[width=\textwidth]{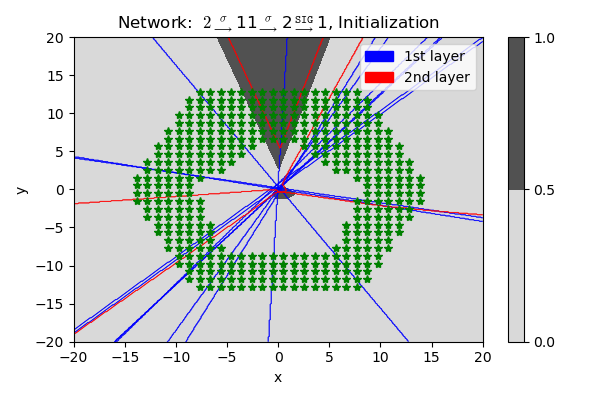}
    \end{subfigure}
    \hfill 
    \begin{subfigure}[b]{0.23\textwidth}
        \centering
        \includegraphics[width=\textwidth]{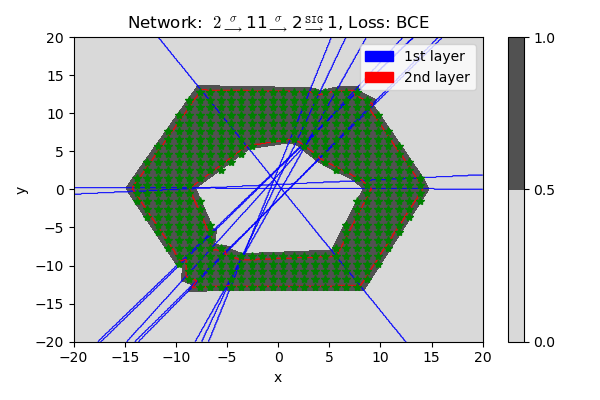}
    \end{subfigure}
    \\
    \begin{subfigure}[b]{0.23\textwidth}
        \centering
        \includegraphics[width=\textwidth]{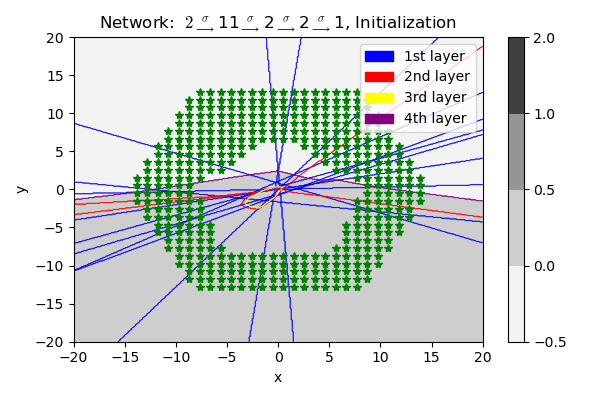}
    \end{subfigure}
    \hfill 
    \begin{subfigure}[b]{0.23\textwidth}
        \centering
        \includegraphics[width=\textwidth]{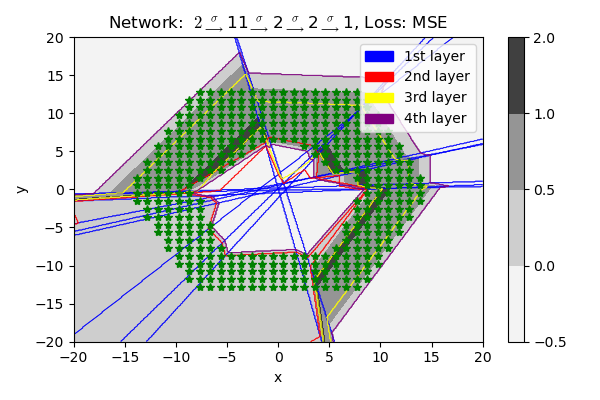}
    \end{subfigure}
    \hfill 
    \begin{subfigure}[b]{0.23\textwidth}
        \centering
        \includegraphics[width=\textwidth]{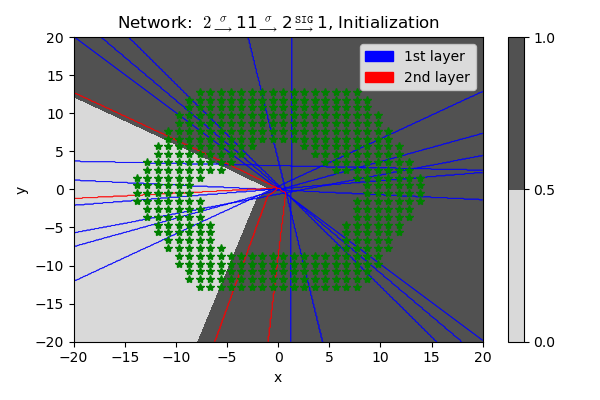}
    \end{subfigure}
    \hfill 
    \begin{subfigure}[b]{0.23\textwidth}
        \centering
        \includegraphics[width=\textwidth]{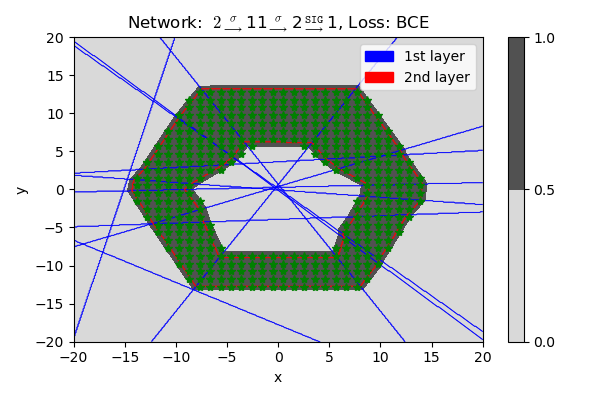}
    \end{subfigure}
    \\
    \begin{subfigure}[b]{0.23\textwidth}
        \centering
        \includegraphics[width=\textwidth]{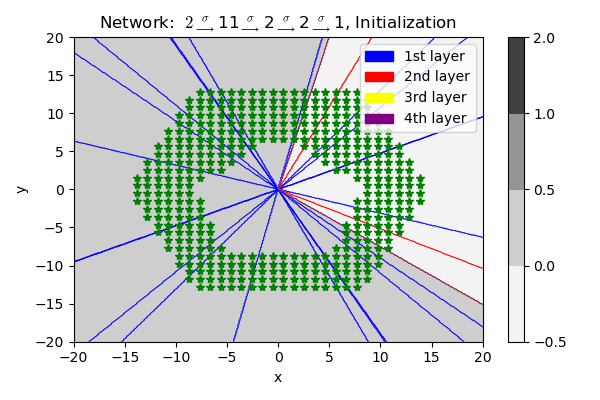}
    \end{subfigure}
    \hfill 
    \begin{subfigure}[b]{0.23\textwidth}
        \centering
        \includegraphics[width=\textwidth]{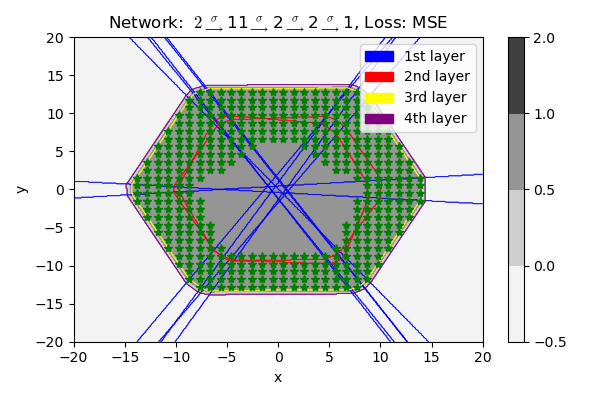}
    \end{subfigure}
    \hfill 
    \begin{subfigure}[b]{0.23\textwidth}
        \centering
        \includegraphics[width=\textwidth]{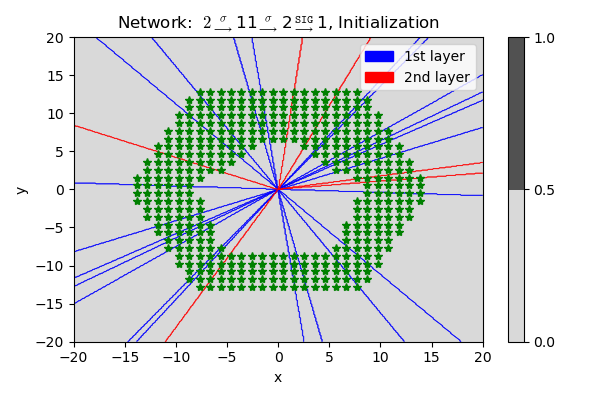}
    \end{subfigure}
    \hfill 
    \begin{subfigure}[b]{0.23\textwidth}
        \centering
        \includegraphics[width=\textwidth]{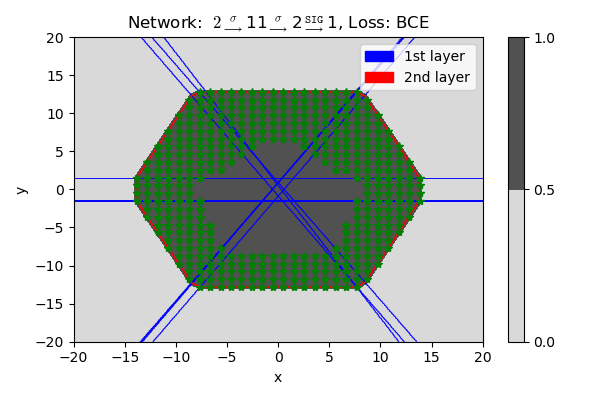}
    \end{subfigure}
    \\
    \begin{subfigure}[b]{0.23\textwidth}
        \centering
        \includegraphics[width=\textwidth]{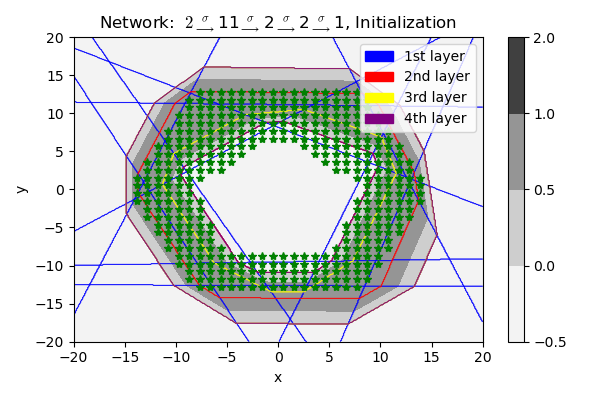}
    \end{subfigure}
    \hfill 
    \begin{subfigure}[b]{0.23\textwidth}
        \centering
        \includegraphics[width=\textwidth]{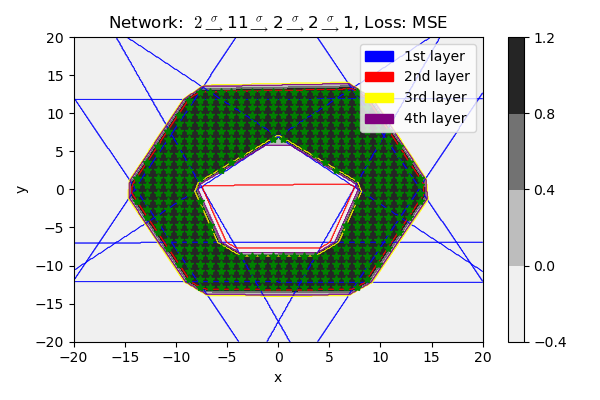}
    \end{subfigure}
    \hfill 
    \begin{subfigure}[b]{0.23\textwidth}
        \centering
        \includegraphics[width=\textwidth]{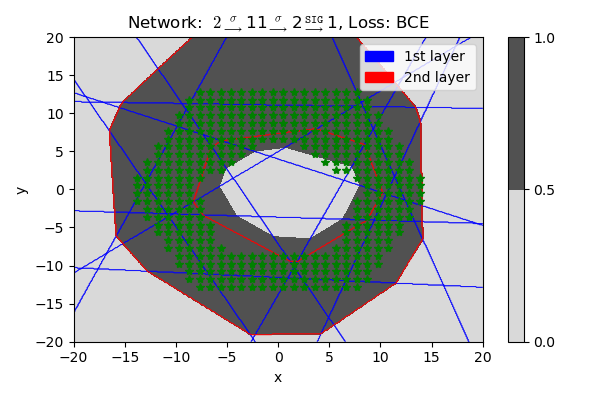}
    \end{subfigure}
    \hfill 
    \begin{subfigure}[b]{0.23\textwidth}
        \centering
        \includegraphics[width=\textwidth]{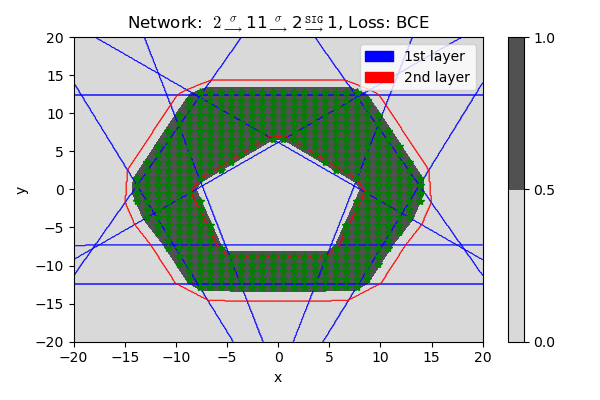}
    \end{subfigure}
    \caption{Initialization and convergence results to approximate $\indicator{\Xc_2}$ via gradient descent. 
    The first and third columns present the network initialization, and the second and fourth columns present the trained network under MSE and BCE loss, respectively.
    Each row has the same initialization scheme, which is given by uniform initialization $U([0,1])$, normal initialization $N(0,\Ib)$, Xavier uniform and normal initialization \cite{glorot2010understanding}, He uniform and normal initialization \cite{he2015delving}, small-norm initialization, and manually initialized weights and bias, from top to bottom.}
    \label{fig: initialization HEX}
\end{figure}

\end{document}

%% file: macros.tex
\theoremstyle{plain}
\newtheorem{theorem}{Theorem}[section]
\newtheorem{proposition}[theorem]{Proposition}
\newtheorem{lemma}[theorem]{Lemma}
\newtheorem{corollary}[theorem]{Corollary}

\theoremstyle{definition}

\newtheorem{example}[theorem]{Example} 
\theoremstyle{remark}
\newtheorem{remark}[theorem]{Remark}

\long\def\comment#1{} 

\newcommand{\xmath}[1] {\ensuremath{#1}\xspace}
\newcommand{\blmath}[1] {\xmath{\bm{#1}}}

\newcommand{\B}{\blmath{B}}

\newcommand{\norm}[1] {\xmath{\left\| #1 \right\|}}

\newcommand{\Ib}{{\blmath I}}

\newcommand{\Wb}{{\blmath W}}

\newcommand{\bb}{{\blmath b}}

\newcommand{\ub}{{\blmath u}}

\newcommand{\wb}{{\blmath w}}
\newcommand{\xb}{{\blmath x}}
\newcommand{\yb}{{\blmath y}}

\newcommand{\Cc}{\mathcal{C}}

\newcommand{\Nc}{\mathcal{N}}

\newcommand{\Tc}{\mathcal{T}}

\newcommand{\Xc}{\mathcal{X}}

\newcommand{\Nd}{{\mathbb N}}
\newcommand{\Rd}{{\mathbb R}}

\newcommand{\zerob}{\mathbf{0}}

\newcommand{\beq}{\begin{equation}}
\newcommand{\eeq}{\end{equation}}
\newcommand{\beqa}{\begin{eqnarray}}
\newcommand{\eeqa}{\end{eqnarray}}

\newcommand{\indicator}[1]{\mathbbm{1}_{\{#1\}}}

\newcommand{\floor}[1]{\left\lfloor#1\right\rfloor}